\documentclass[letterpaper]{article}

\usepackage{conf}
\usepackage{times}
\usepackage{helvet}
\usepackage{courier}
\usepackage[hyphens]{url}
\usepackage{graphicx}
\usepackage{natbib}
\usepackage{caption}
\usepackage{algorithm}
\usepackage[noend]{algpseudocode}
\usepackage{eqparbox}
\usepackage{subfigure}

\usepackage{newfloat}
\usepackage{listings}
\DeclareCaptionStyle{ruled}{labelfont=normalfont,labelsep=colon,strut=off}
\floatstyle{ruled}
\newfloat{listing}{tb}{lst}{}
\floatname{listing}{Listing}

\usepackage{xspace}
\usepackage{color}
\usepackage{enumitem}
\usepackage{tabularx}
\usepackage{tablefootnote}
\usepackage{threeparttable}
\usepackage{booktabs} 
\usepackage{amsfonts}
\usepackage{amsmath}
\usepackage{amsthm}
\usepackage{amssymb}
\usepackage{mathtools}
\usepackage{amsmath,amssymb,graphicx}
\usepackage{multirow}
\usepackage[skip=0.5\baselineskip]{caption}
\usepackage{dblfloatfix}

\usepackage{array}
    \newcolumntype{P}[1]{>{\centering\arraybackslash}p{#1}}
    \newcolumntype{M}[1]{>{\centering\arraybackslash}m{#1}}

\makeatletter
\newtheorem{problem}{Problem}

\newtheorem{lemma}{Lemma}
\newtheorem{theorem}{Theorem}

\algdef{SE}[FUNCTION]{Function}{EndFunction}
   [2]{\algorithmicfunction\ \textproc{#1}\ifthenelse{\equal{#2}{}}{}{(#2)}}
   {\algorithmicend\ \algorithmicfunction}

\newcommand{\reflect@squig}[2]{
  \reflectbox{$\m@th#1\rightsquigarrow$}
}
\providecommand{\leftlsquigarrow}{
  \mathrel{\mathpalette\reflect@squig\relax}
}

\newcommand{\mat}[1]{\mathbf{#1}}

\newcommand{\vect}[1]{\mathbf{#1}}

\newcommand{\DTDG}{\mathcal{G}} 
\newcommand{\G}[1]{\mathcal{G}_{#1}} 
\newcommand{\V}[1]{\mathcal{V}_{#1}} 
\newcommand{\E}[1]{\mathcal{E}_{#1}} 
\newcommand{\A}[1]{\mat{A}_{#1}} 
\newcommand{\Ahat}[1]{\mat{\hat{A}}_{#1}} 
\newcommand{\Dhat}[1]{\mat{\hat{D}}_{#1}} 
\newcommand{\Atilde}[1]{\mat{\tilde{\mathbfcal{A}}}_{#1}} 
\newcommand{\matH}[1]{\mat{H}_{#1}} 
\newcommand{\I}[1]{\mathbfcal{I}_{#1}}
\newcommand{\F}[1]{\mat{F}_{#1}} 
\newcommand{\X}[1]{\mathbfcal{X}_{#1}} 
\newcommand{\Xhat}[1]{\mathbfcal{\hat{X}}_{#1}} 
\newcommand{\Xtilde}[1]{\mathbfcal{\tilde{X}}_{#1}} 
\newcommand{\matS}[1]{\mathbfcal{S}_{#1}} 
\newcommand{\T}[1]{\mathbfcal{T}_{#1}} 
\newcommand{\lmd}[1]{\boldsymbol{\lambda}_{#1}} 
\newcommand{\Lrwr}[1]{\mathbfcal{L}_{#1}^{\text{\normalfont rwr}}}
\newcommand{\matL}[1]{\mat{L}_{#1}} 

\newcommand{\wt}{{\color{white}{\scriptsize${}^\blacktriangle$}}}
\newcommand{\rt}{{\color{red}{\scriptsize${}^\blacktriangle$}}}
\newcommand{\bt}{{\color{blue}\scriptsize${}^\blacktriangledown$}}

\newcommand{\lwt}{{\color{white}{\footnotesize${}^\blacktriangle$}}}
\newcommand{\lrt}{{\color{red}{\footnotesize${}^\blacktriangle$}}}
\newcommand{\lbt}{{\color{blue}\footnotesize${}^\blacktriangledown$}}

\newcommand{\PI}[1]{\vect{x}_{#1}} 

\newcommand{\method}{\textsc{TiaRa}\xspace}

\makeatother

\nocopyright 

\DeclareMathAlphabet\mathbfcal{OMS}{cmsy}{b}{n}

\begin{document}

\title{Time-aware Random Walk Diffusion to Improve Dynamic Graph Learning}
\author {
    Jong-whi Lee and
    Jinhong Jung
}
\affiliations {
    Department of Computer Science and Engineering, Jeonbuk National University, South Korea \\
    jong.whi.lee@jbnu.ac.kr, jinhongjung@jbnu.ac.kr
}

\maketitle

\begin{abstract}
How can we augment a dynamic graph for improving the performance of dynamic graph neural networks? Graph augmentation has been widely utilized to boost the learning performance of GNN-based models. However, most existing approaches only enhance spatial structure within an input static graph by transforming the graph, and do not consider dynamics caused by time such as temporal locality, i.e., recent edges are more influential than earlier ones, which remains challenging for dynamic graph augmentation.
In this work, we propose \method (Time-aware Random Walk Diffusion), a novel diffusion-based method for augmenting a dynamic graph represented as a discrete-time sequence of graph snapshots. For this purpose, we first design a time-aware random walk proximity so that a surfer can walk along the time dimension as well as edges, resulting in spatially and temporally localized scores. We then derive our diffusion matrices based on the time-aware random walk, and show they become enhanced adjacency matrices that both spatial and temporal localities are augmented. Throughout extensive experiments, we demonstrate that \method effectively augments a given dynamic graph, and leads to significant improvements in dynamic GNN models for various graph datasets and tasks.

\end{abstract}

\section{Introduction}
Dynamic graphs represent various real-world relationships that dynamically occur over time.
Learning such dynamic graphs has recently attracted considerable attention from machine learning communities~\cite{skarding2021foundations,han2021dynamic}, and plays a crucial role in diverse applications such as link prediction~\cite{yang2021discrete,ParejaDCMSKKSL20},
node or edge classification~\cite{xu2019spatio,ParejaDCMSKKSL20}, time-series traffic forecasting~\cite{wu2020connecting,guo2019attention}, knowledge completion~\cite{DBLP:conf/kdd/JungJK21}, and pandemic forecasting~\cite{panagopoulos2021transfer}.
Over the last years, many researchers have put tremendous effort into developing interesting methods by sophisticatedly fusing GNNs and recurrent neural networks (RNN) or attention mechanisms for continuous-time~\cite{XuRKKA20,tgnicmlgrl2020} and discrete-time~\cite{seo2018structured,ParejaDCMSKKSL20,yang2021discrete} dynamic graphs.

With the astonishing progress of GNNs, diverse augmentation techniques~\cite{zhao2022graph,DBLP:conf/www/YooSK22} have been proposed to increase the generalization power of GNN models, especially on a static graph.
Previous approaches mainly transform the topological structure of the input graph.
For example, drop-based methods stochastically remove a certain number of edges~\cite{Rong2020DropEdge} or nodes~\cite{feng2020graph} at each training epoch in a similar manner to dropout regularization.
On the contrary, diffusion methods~\cite{klicpera2019diffusion} insert additional edges having weights scored by graph diffusions such as Personalized PageRank~\cite{tong2006fast}, thereby augmenting a spatial locality around each node and improving graph convolution.

However, the aforementioned techniques assume to augment data within a static graph, and dynamic graph augmentation problem has not yet been comprehensively studied.
Unlikely static graphs, dynamic graphs change or evolve over time by their nature; thus, dynamic graph augmentation needs to simultaneously consider temporal dynamics as well as spatial structure.
More specifically, as verified in previous works~\cite{tgnicmlgrl2020,shin2017wrs,DBLP:journals/vldb/LeeSF20},
real-world dynamic graphs exhibit \textit{temporal locality} indicating that graph objects such as nodes and triangles tend to be more affected by more recent edges than older ones,
i.e., edges closer to a specific object in time are more likely to provide important information.
Naively applying a static augmentation method to each time step cannot consider such a temporal locality.

In this work, we propose \method (\underline{Ti}me-\underline{a}ware \underline{Ra}ndom Walk Diffusion), a novel diffusion-based augmentation method for a discrete-time dynamic graph which is represented by a temporal sequence of graph snapshots.
\method aims to augment both spatial and temporal localities of each graph snapshot.
For this purpose, we design a time-aware random walk that a surfer randomly moves around nodes or a time-axis to measure spatially and temporally localized scores.
We then derive time-aware random walk diffusion from the scores, and interpret it as the combination of spatial and temporal augmenters.
Our diffusion matrices are used as augmented adjacency matrices for any dynamic GNN models in discrete-time domain.
We further adopt approximate techniques such as power iteration and sparsification to reduce a heavy cost for computing the diffusion matrices.

Our contributions are summarized as follows:
\setlist[itemize]{left=5mm}
\begin{itemize}
    \item {\textbf{Method.}
        We propose \method, a novel and model-agnostic method for dynamic graph augmentation using time-aware random walks.
        \method strengthens not only a spatial locality but also a temporal locality of a dynamic graph so that dynamic GNNs perform better.
    }
    \item {\textbf{Analysis.}
        We analyze how \method augments both spatial and temporal localities (Theorem~\ref{theorem:trwr_diff_closed}) and complexities of \method (Theorem~\ref{theorem:comp}) in real dynamic graphs.
    }
    \item {\textbf{Experiments.}
        We demonstrate that \method effectively augments a given dynamic graph, and leads to consistent improvements in GNNs for temporal link prediction and node classification tasks.
    }
\end{itemize}

The code of \method and the datasets are publicly available at \textbf{https://github.com/dev-jwel/TiaRa}.

\section{Related Work}
\textbf{Augmentation for Static GNNs.}
Graph augmentation~\cite{zhao2022graph} aims to reduce over-fitting for training GNN models by modifying an input graph.
DropEdge~\cite{Rong2020DropEdge} or DropNode~\cite{feng2020graph} randomly drop edges or nodes at each epoch.
These augment the diversity of the input graph by creating different copies sampled from the graph.
GDC~\cite{klicpera2019diffusion} adds new edges weighted by a graph diffusion derived from node proximities.
GDC boosts a spatial locality of the graph so that a GNN can consider adjacent nodes as well as distant ones during their convolutions, enhancing its representation power.
However, most of existing methods are limited to augment dynamic graphs because they do not consider temporal properties.

\textbf{GNNs and Augmentation for Dynamic Graphs.}
Dynamic graphs~\cite{kazemi2020representation} are categorized as: discrete-time dynamic graphs (DTDG) and continuous-time dynamic graphs (CTDG) where a DTDG is represented as a sequence of graph snapshots with multiple discrete time steps while a CTDG is represented as a set of temporal edges whose time-stamps have continuous values.
It is straightforward to convert a CTDG to a DTDG by distributing the continuous-time edges into multiple bins in chronological order, but the reverse is not possible because continuous-time values are generally lacked in most DTDGs~\cite{yang2021discrete}, i.e., models for DTDGs can be applied to CTDGs, but the reverse is rather limited.
Hence, we narrow our focus to representation learning on DTDGs.

Dynamic GNNs have rapidly advanced under the framework that closely integrates GNNs and temporal sequence models such as RNNs to capture spatial and temporal relations on dynamic graphs~\cite{skarding2021foundations}.
GCRN~\cite{seo2018structured} uses a GCN to produce node embeddings on each graph snapshot, and then forwards them to an LSTM for modeling temporal dynamics.
STAR~\cite{xu2019spatio} utilizes a GRU combined with spatial and temporal attentions.
DySat~\cite{sankar2020dysat} employs a self-attention strategy to aggregate  spatial neighborhood and temporal dynamics.
EvolveGCN~\cite{ParejaDCMSKKSL20} evolves the parameters of GCNs using RNNs.
To consider hierarchical properties in real graphs, HTGN~\cite{yang2021discrete} extends the framework to hyperbolic space.

As a related method, MeTA~\cite{wang2021adaptive} adaptively augments a temporal graph based on predictions of a temporal graph network, which perturbs time and removes or adds edges.
However, it is difficult to employ MeTA for the aforementioned DTDG models because MeTA is  designed for CTDGs requiring continuous-time values.

\section{Preliminaries}
\textbf{Random Walk with Restart (RWR).}
Our work is related to RWR which measures node similarity scores that are spatially localized to seed node $s$~\cite{nassar2015strong}, i.e., scores of nearby nodes highly associated with $s$ are high while those of distant nodes are low.
Diffusion methods such as GDC exploit RWR to augment a spatial locality.

Let $\PI{s}$ be a vector of RWR scores w.r.t. the seed node $s$.
Given a row-normalized adjacency matrix $\Atilde{}$ and a restart probability $\alpha$, the vector $\PI{s}$ is represented as follows:
\begin{align*}
    \PI{s} = (1-\alpha)\Atilde{}^{\top}\PI{s} + \alpha\vect{i}_{s} \Leftrightarrow
    \PI{s} = \alpha\matL{}^{-1}\vect{i}_{s} \Leftrightarrow \PI{s} = \Lrwr{}\vect{i}_{s}
\end{align*}
where $\matL{} = \I{n} - (1-\alpha)\Atilde{}^{\top}$ is the random-walk normalized Laplacian matrix, and $\vect{i}_{s}$ is the $s$-th unit vector.
Notice that $\Lrwr{} = \alpha \matL{}^{-1}$ is a column-stochastic transition matrix interpreted as a diffusion kernel that diffuses a given distribution such as $\vect{i}_{s}$ on the graph through RWR.

\textbf{Problem Formulation.}
A discrete-time dynamic graph (DTDG) $\DTDG$ is represented as a sequence $\{\G{1},\cdots,\G{T}\}$ of snapshots in a chronological order where $T$ is the number of time steps~\cite{skarding2021foundations}.
Each snapshot $\G{t}=(\V{}, \E{t}, \F{t})$ is a graph with a shared set $\V{}$ of nodes and a set $\E{t}$ of edges at time $t$ where $n=|\V{}|$ is the number of nodes.
$\F{t}\in \mathbb{R}^{n \times d}$ is an initial node feature matrix where $d$ is a feature dimension, and $\A{t} \in \mathbb{R}^{n \times n}$ denotes the sparse and self-looped adjacency matrix of $\G{t}$.
The node representation learning on the dynamic graph $\DTDG$ aims to learn a function $\mathcal{F}_{\Theta}(\cdot)$ parameterized by $\Theta$ and produce hidden node embeddings $\matH{t} \in \mathbb{R}^{n \times d}$ for each time $t$, represented as:
\begin{equation}
    \label{eq:dtdg_gnn}
    \matH{t} = \mathcal{F}_{\Theta}(\Atilde{t}, \F{t}, \matH{t-1})
\end{equation}
where $\Atilde{t}$ is a normalized adjacency matrix of $\A{t}$, and $\matH{t-1}$ contains the latest hidden embeddings before time $t$.
The above framework of Equation~\eqref{eq:dtdg_gnn} is generally adopted in  existing methods~\cite{seo2018structured,ParejaDCMSKKSL20,yang2021discrete} for learning DTDGs where $\mathcal{F}_{\Theta}(\cdot)$ is usually designed by the combination of GNNs and RNNs.

\begin{figure*}[t]
    \centering
    \includegraphics[width=1.0\linewidth]{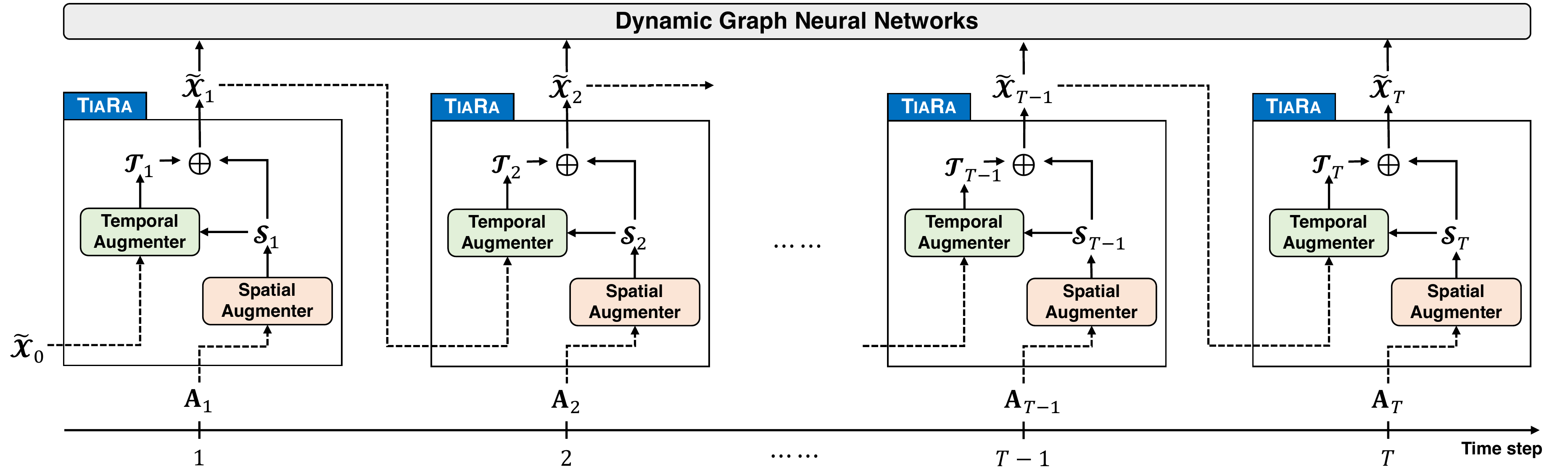}
    \caption{
        \label{fig:overview}
        Overall architecture of \method. Given the adjacency matrix $\A{t}$ at time $t$, \method outputs a time-aware random walk diffusion matrix $\Xtilde{t}$ combined with spatial augmenter $\matS{t}$ and temporal augmenter $\T{t}$ after sparsification.
    }
\end{figure*}

\begin{problem}[Dynamic Graph Augmentation]
\label{prob:aug}
Given a temporal sequence $\{\A{1}, \cdots, \A{T}\}$ of $\DTDG$, the problem is to generate a sequence of new adjacency matrices improving the performance of a model $\mathcal{F}_{\Theta}(\cdot)$. \hfill\qedsymbol
\end{problem}

\section{Proposed Method}
We depict the overall framework of \method in Figure~\ref{fig:overview}.
Given $\{\A{1}, \cdots, \A{T}\}$ of a dynamic graph $\DTDG$, our \method aims to produce a time-aware random walk diffusion matrix $\Xtilde{t}\in \mathbb{R}^{n \times n}$ for each time step $t$ using two diffusion based modules, called \textit{spatial} and \textit{temporal} augmenters.

The spatial augmenter enhances a spatial locality of $\A{t}$ using random walks, resulting in a spatial diffusion matrix $\matS{t}$.
The temporal augmenter receives the previous $\Xtilde{t-1}$ that contains information squashed from the initial time to $t-1$, and then disseminates it through $\matS{t}$ at the current $t$.
This leads to a temporal diffusion matrix $\T{t}$ in which a temporal locality is magnified.
Finally, \method linearly combines $\matS{t}$ and $\T{t}$, and sparsifies to form $\Xtilde{t}$.
We replace each adjacency matrix $\A{t}$ with $\Xtilde{t}$ for the inputs of dynamic GNN models.
If necessary, we simply use edges of the graph represented by $\Xtilde{t}$ without weights, or make the graph undirected by using $(\Xtilde{t}+\Xtilde{t}^{\top})/2$ after the sparsification.

\subsection{Time-aware Random Walk with Restart}
It is limited to directly employ RWR in a dynamic graph because RWR measures only spatially localized scores in a single static graph.
In this section, we extend RWR to Time-aware RWR (TRWR) so that TRWR produces node-to-node scores which are spatially and temporally localized.

One idea for TRWR is to virtually connect identical nodes from $\G{t}$ to $\G{t+1}$ for each time step $t$~\cite{DBLP:conf/kdd/0002S021} as shown in Figure~\ref{fig:trwr_idea}.
Then, a random surfer not only moves around the current $\G{t}$ but also jumps to the next $\G{t+1}$; thus, the surfer becomes time-aware.
In the beginning, the surfer starts from a seed node $s$ at the initial time step (e.g., $t=1$).
After a few movements, suppose the surfer is at node $u$ in $\G{t}$.
Then, it takes one of the following actions:
\setlist[itemize]{left=2mm}
\begin{itemize}
    \item {
        \textbf{Action 1) Random walk.}
        The surfer randomly moves to one of the neighbors from node $u$ in the current graph $\G{t}$ with probability $1 - \alpha - \beta$.
    }
    \item {
        \textbf{Action 2) Restart.}
        The surfer goes back to the seed node $s$ in $\G{t}$ with probability $\alpha$.
    }
    \item {
        \textbf{Action 3) Time travel.}
        The surfer does time travel from node $u$ in $\G{t}$ to that node in $\G{t+1}$ with probability $\beta$.
    }
\end{itemize}
where $\alpha$ and $\beta$ are called restart and time travel probabilities, respectively, and $0 < \alpha + \beta < 1$.
Note that we do not allow the surfer to move backward from $\G{t+1}$ to $\G{t}$ because the future information at time $t+1$ should be prevented when we make a prediction at time $t$.

\begin{figure}[t]
    \centering
    \includegraphics[width=0.9\linewidth]{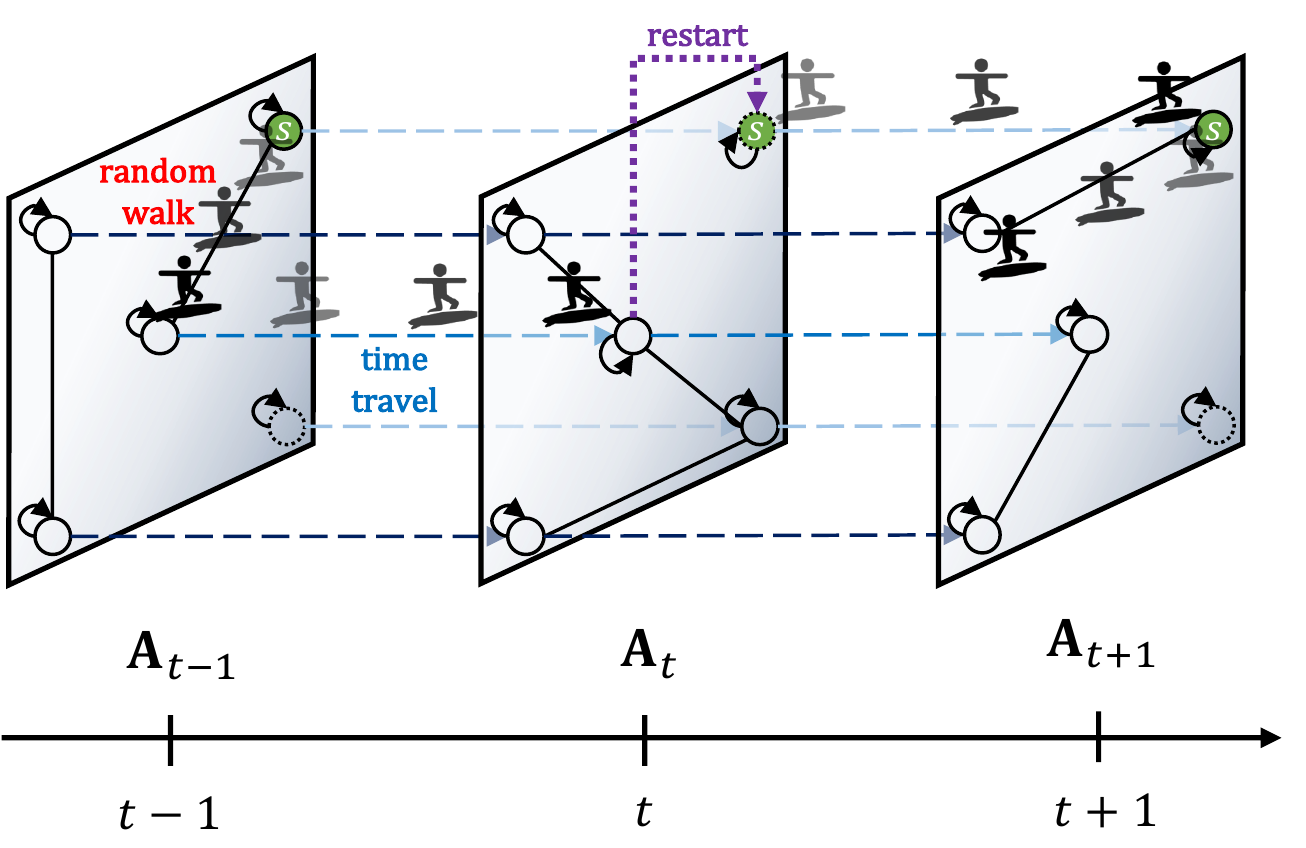}
    \caption{
        \label{fig:trwr_idea}
        Illustration of how a time-aware random surfer moves around a dynamic graph where dashed arrows indicate virtually connected edges along the time-axis.
    }
\end{figure}

Through TRWR, the vector $\PI{t} \in \mathbb{R}^{n}$ of stationary probabilities that the surfer visits each node from the seed node $s$ in $\G{t}$ is recursively represented as follows:
\begin{equation}
    \label{eq:trwr}
    \PI{t,s}
    =
    \underbrace{(1 - \alpha - \beta) \Atilde{t}^{\top} \PI{t,s}}_{\text{Random walk}}
    +
    \underbrace{\alpha\vect{i}_{s}}_{{\large\phantom{1}}\text{Restart}{\large\phantom{1}}}
    +
    \underbrace{\beta \PI{t-1,s}}_{\text{Time travel}}
\end{equation}
where $\vect{i}_{s}$ is the $s$-th unit vector of size $n$.
$\Atilde{t}$ is a row-normalized matrix of $\A{t}$ (i.e., $\Atilde{t}=\mat{D}_{t}^{-1}\A{t}$ where $\A{t}$ is a self-looped adjacency matrix and $\mat{D}_{t}$ is a diagonal out-degree matrix of $\A{t}$).
If $t=0$, we define $\PI{0,s}$ as $\vect{i}_{s}$.

In the above equation, the random walk part propagates scores of $\PI{t,s}$ over ${\Atilde{t}}$.
The restart part makes the scores spatially localized around the seed node $s$, which is controlled by $\alpha$.
The time travel part injects scores of the previous $\PI{t-1,s}$ to make $\PI{t,s}$ temporally localized, which is controlled by $\beta$.
Notice TRWR extends RWR to a discrete-time dynamic graph, i.e., $\beta=0$ leads to RWR scores on each graph snapshot without considering temporal information.

\subsection{Time-aware Random Walk Diffusion Matrices}
In Equation~\eqref{eq:trwr}, $\PI{t, s}\in \mathbb{R}^{n \times 1}$ is a column vector of a probability distribution w.r.t. a seed node $s$.
For all seeds $s \in \V{}$, we horizontally stack $\{\PI{t, s}\}$ to form $\X{t}\in \mathbb{R}^{n \times n}$ such that $\PI{t, s}$ is the $s$-th column of $\X{t}$, i.e., $\PI{t, s} = \X{t}\vect{i}_{s}$.
We call $\X{t}$ a time-aware random walk diffusion matrix at time $t$.
The derivation of $\X{t}$ starts by moving the term of the random walk to the left side in Equation~\eqref{eq:trwr} as follows:
\begin{align*}
    \begin{split}
    \left( \I{n} - (1-\alpha-\beta) \Atilde{t}^{\top} \right) \PI{t,s} = \alpha\vect{i}_{s} + \beta\PI{t-1,s}
    \end{split}
\end{align*}

Let $\matL{t} \coloneqq \I{n} - (1 - \alpha - \beta) \Atilde{t}^{\top}$ where $\I{n}$ is an $n \times n$ identity matrix, and $\PI{t-1, s} = \X{t-1}\vect{i}_{s}$ as described above.
Thus, $\PI{t,s}$ is written as the following:
\begin{align}
    \PI{t,s} &= \matL{t}^{-1}\left(\alpha\vect{i}_{s} + \beta\X{t-1}\vect{i}_{s}\right) \nonumber \\
             &= \left(\alpha\matL{t}^{-1}\I{n} + \beta\matL{t}^{-1}\X{t-1}\right)\vect{i}_{s} = \X{t}\vect{i}_{s} \label{eq:trwr_diff}
\end{align}
where $\X{t}=\alpha \matL{t}^{-1}\I{n} + \beta \matL{t}^{-1}\X{t-1}$ for $t>0$, and $\X{0}=\I{n}$ because $\PI{0,s}$ is defined as $\vect{i}_{s}$.

\textbf{Spatial and Temporal Augmenters.}
We obtain the recurrence relation of $\X{t}$ from Equation~\eqref{eq:trwr_diff}, and further rearrange it to interpret the process as follows:
\begin{align*}
    \begin{split}
    \X{t} &= \alpha\matL{t}^{-1}\I{n} + \beta\matL{t}^{-1}\X{t-1} \\
          &= \frac{\alpha}{\alpha+\beta} \left[(\alpha+\beta)\matL{t}^{-1}\I{n}\right]
          + \frac{\beta}{\alpha+\beta} \left[(\alpha+\beta)\matL{t}^{-1}\X{t-1}\right]
    \end{split}
\end{align*}

In the above, we set $\Lrwr{t} = (\alpha + \beta) \matL{t}^{-1}$ which is the diffusion kernel by RWR on the graph $\G{t}$ where its restart probability is $\alpha + \beta$.
Let $\gamma = \beta/(\alpha + \beta)$ where $0 < \gamma < 1$; then, $\X{t}$ is represented as follows:
\begin{align}
    \begin{split}
    \X{t} &= (1-\gamma)\underbrace{\left(\Lrwr{t}\I{n}\right)}_{\matS{t}} + \gamma\underbrace{\left(\Lrwr{t}\X{t-1}\right)}_{\T{t}=\matS{t}\X{t-1}} \label{eq:trwr_diff_rec}
    \end{split}
\end{align}
where $\matS{t}$ is a spatial diffusion matrix, and $\T{t}$ is a temporal diffusion matrix.

The meaning of $\matS{t}$ is the result of diffusing the $s$-th column $\vect{i}_{s}$ of $\I{n}$ through $\Lrwr{t}$ for each node $s$.
This is interpreted as the augmentation of a spatial locality of each node through RWR within $\G{t}$.
On the other hand, $\T{t}$ is the result of diffusing $\PI{s, t-1}$ of $\X{t-1}$ through $\Lrwr{t}$ for each node $s$.
Note that $\PI{s,t-1}$ contains the probabilities that the surfer visits each node starting from node $s$ during the travel from the initial time to $t-1$.
Thus, it spreads the past proximities of $\PI{s,t-1}$ in the current $\G{t}$ through $\Lrwr{t}$, which consequently reflects the temporal information to $\G{t}$.

The final diffusion matrix $\X{t}$ is a convex combination between $\matS{t}$ and $\T{t}$ w.r.t. $\gamma$, which is denoted by $\oplus$ in Figure~\ref{fig:overview}.
Notice that $\X{t}$ is a column stochastic transition matrix for every time step $t$, which is proved in Lemma~\ref{lemma:transition}, implying that as an augmented adjacency matrix, $\Atilde{t}$ can be replaced with $\X{t}^{\top}$ for the input of GNNs in Equation~\eqref{eq:dtdg_gnn}.

\textbf{Interpretation.}
We further analyze how $\X{t}$ reflects the spatial and temporal information of the input dynamic graph.
For this purpose, we first obtain the closed-form expression of $\X{t}$ which is described in Theorem~\ref{theorem:trwr_diff_closed}.
\begin{theorem}
\label{theorem:trwr_diff_closed}
The closed-form expression of $\X{t}$ is:
\begin{align}
    \label{eq:trwr_diff_closed}
    \X{t} = (1-\gamma)\left( \sum_{i = 0}^{t-2} \gamma^{i} \Lrwr{t \leftlsquigarrow t-i} \right) + \gamma^{t-1}\Lrwr{t\leftlsquigarrow1}
\end{align}
where $\Lrwr{j \leftlsquigarrow i}\!=\!\Lrwr{j} \Lrwr{j-1} \!\cdots\! \Lrwr{i}$ for $j\!>\!i$, and $\Lrwr{i \leftlsquigarrow i}\!=\!\Lrwr{i}$.
\end{theorem}
\begin{proof}
It is proved by mathematical induction, and the detailed proof is described in Appendix~\ref{sec:appendix:proof}.
\end{proof}

In the theorem, $\Lrwr{j \leftlsquigarrow i}$ indicates a random walk diffusion traveling from former time $i$ toward latter time $j$.
According to Equation~\eqref{eq:trwr_diff_closed}, $\X{t}$ is concisely represented as:
{
\begin{align*}
    \X{t}\!\propto\!\underbrace{\overbrace{
        \gamma^{0}\Lrwr{t}+\gamma^{1}\Lrwr{t \leftlsquigarrow t-1} +\!\cdots\!+ \gamma^{t-2}\Lrwr{t \leftlsquigarrow 2}+\frac{\gamma^{t-1}}{1\!-\!\gamma}\Lrwr{t \leftlsquigarrow 1} 
    }^{
        \text{Augmentation of spatial and temporal localities} 
    }}_{
        \text{$\Longleftarrow$ Emphasized {\color{white}$\cdots\cdots\cdots\cdots\cdots\cdots\cdots\cdots\cdots\cdots\cdots$} Decayed $\Longrightarrow$} 
    }
\end{align*}
}

Note that $\X{t}$ is more affected by the information close to time $t$ than that passed from the distant past.
The influence of $\Lrwr{t \leftlsquigarrow k}$ is decayed by $\gamma$ as time $k$ is further away from time $t$, while it is emphasized as time $k$ is near to time $t$ where the ratio $\gamma$ is interpreted as a \textit{temporal decay ratio}.
This explanation is consistent with the temporal locality, i.e., the tendency that recent edges are more influential than older ones.
Combined with the spatial diffusion $\Lrwr{t}$, the result of $\X{t}$ augments both spatial and temporal localities in $\G{t}$.

\textbf{Discussion.}
\method is a generalized version of GDC with PPR kernel to dynamic graphs since \method with $\beta=0$ spatially augments data within a single $\G{t}$ at each time, which is exactly what GDC does.
However, GDC does not consider temporal information for its augmentation, and it performs worse than \method as shown in Tables~\ref{tab:link_auc} and ~\ref{tab:node_f1}.

\subsection{Algorithm for \method}
Most graph diffusions involve heavy computational cost, especially for a large graph, and result in a dense matrix.
The computation of $\X{t}$ also exhibits the same issue, and thus we adopt approximate techniques to alleviate the problem.
Including the approximate strategies, the procedure of \method is summarized in Algorithm~\ref{alg:method} where $\Xtilde{0}$ is set to $\I{n}$.

\textbf{Power iteration.} The main bottleneck for obtaining $\X{t}$ is to compute the matrix inversion $\matL{t}^{-1}$ of $\Lrwr{t}$ in Equation~\eqref{eq:trwr_diff_rec}, which requires $O(n^3)$ time.
Instead of directly calculating the inversion, we use power iteration (lines~$9\!\sim\!15$) based on the following~\cite{yoon2018tpa}:
\begin{align*}
    \matL{t}^{-1}
    =
    \sum_{k=0}^{\infty} c^{k} \left(\Atilde{t}^{\top}\right)^{k}
    \approxeq
    \sum_{k=0}^{K} c^{k} \left(\Atilde{t}^{\top}\right)^{k}
\end{align*}
where $c = 1 - \alpha - \beta$, and $K$ is the number of iterations.
Let $\mat{M}_{t}^{(k)}$ be the result after $k$ iterations; then, it is recursively represented as follows:
\begin{align*}
    \mat{M}_{t}^{(k)} = \I{n} + c\Atilde{t}^{\top}\mat{M}_{t}^{(k-1)}
\end{align*}
where $\mat{M}_{t}^{(0)} = \I{n}$ and $\mat{M}_{t}^{(K)} \approxeq \matL{t}^{-1}$.
Note $\Atilde{t}$ is a normalized adjacency matrix (line 1) in which self-loops are added, as traditional GNNs usually do.
The approximate error is bounded by $c^{k}$, and converges to $0$ as $k \rightarrow \infty$ (see Lemma~\ref{lemma:error}).
After that, we set $\Lrwr{t}\!\leftarrow\!(1-c)\mat{M}_{t}^{(K)}$ (line~13).

At a glance, each iteration seems to take $O(n^{3})$ time for the matrix multiplication, but it is much faster than that since each snapshot $\G{t}$ is sparse in most cases.
More specifically, only a few nodes form edges at each time step in real graphs.
We call such nodes \textit{activated} where $\V{t}$ is the set of activated nodes at time $t$, and $n_t = |\V{t}|$.
In each $\G{t}$, a surfer can move only between activated nodes, i.e., only pairs of nodes in $\V{t}$ are diffused.
As seen in Table~\ref{tab:data}, the average $\bar{n}_{t}$ of $n_t$ over time is smaller than $n$ except the Brain dataset.

This allows us to do the power iteration on the sub-matrix $\Atilde{\V{t}} \in \mathbb{R}^{n_t \times n_t}$ of $\Atilde{t}$ for nodes in $\V{t}$ where $m_t$ is the number of non-zeros of $\Atilde{\V{t}}$.
Then, an iteration takes $O(m_tn_t)$ time for a sparse matrix multiplication.
Note $m_t$ is linearly proportional to $n_t$ in real graphs, i.e., $m_t = C_tn_t$ where $C_t$ is a constant.
Let $\bar{m}_t$ be the average number of edges over time.
As seen in Table~\ref{tab:data}, $\bar{C}_{t} = \bar{m}_t/\bar{n}_t$ is smaller than $\bar{n}_t$.
Thus, each iteration takes $O(n_t^2)$ time in average;
overall, it takes $O(n_t^2K + nK)$ time and $O(n_t^2 + n)$ space for $\Lrwr{t}$ (as only $n_t$ nodes are diffused).
More details are provided in Appendix~\ref{sec:appendix:proof}.

\textbf{Sparsification.}
Another bottleneck is that $\X{t}$ is likely to be dense by repeatedly multiplying $\matS{t}\X{t-1}$ (line 4) as time $t$ increases where $\matS{t}=\Lrwr{t}$.
This could be problematic in terms of space as well as running time, especially for graph convolutions since $\X{t}$ is used as an adjacency matrix.
To alleviate this issue, we adopt a sparsification technique suggested in~\cite{klicpera2019diffusion}.
As established in Theorem~\ref{theorem:trwr_diff_closed}, the graph structure of $\X{t}$ is spatially and temporally localized, which allows us to drop small entries of $\X{t}$, resulting in the sparse $\Xtilde{t}$.
For this, we use a filtering threshold $\epsilon$ to set values of $\X{t}$ below $\epsilon$ to zero (line 6).
This strategy has two advantages.
First, it keeps $\Xtilde{t}$ sparse at each time.
Second, it reduces the cost for processing $\matS{t}\Xtilde{t-1}$ as $\matS{t}$ and $\Xtilde{t-1}$ are sparse.
After the sparsification, we normalize $\Xtilde{t-1}$ (line 7) column-wise.
As shown in Figure~\ref{fig:effect_eps}, this sparsification makes the augmentation process fast and lightweight with tiny errors while it does not harm predictive accuracy too much, or can even improve.

\begin{theorem}[Complexity Analysis] \label{theorem:comp}
For each time step $t$, \method takes $O(n_tn/\epsilon + n_t^2K)$ time on average, and produces $\Xtilde{t}$ consuming $O(n/\epsilon)$ space  where $n$ is the number of total nodes, $n_t$ is the number of activated nodes at time $t$, $K$ is the number of iterations, and $\epsilon$ is a filtering threshold.
\end{theorem}
\begin{proof}
The proof is provided in Appendix~\ref{sec:appendix:proof}.
\end{proof}

\begin{algorithm}[t!]
	\small
    \begin{algorithmic}[1]
    	\captionsetup{font=small}
        \caption{\method at time $t$}
        \label{alg:method}
        \Require adjacency matrix $\A{t}$, previous time-aware diffusion matrix $\Xtilde{t-1}$, restart probability $\alpha$, time travel probability $\beta$, number $K$ of iterations, filtering threshold $\epsilon$
        \Ensure time-aware diffusion matrix $\Xtilde{t}^{\top}$
            \State $\Atilde{t} \leftarrow \mat{D}_{t}^{-1}\A{t}$ where $\mat{D}_{t} = \text{diag}$($\A{t}\vect{1}$)
            \State $\Lrwr{t} \leftarrow $ \textsc{Power-Iteration}($\Atilde{t}$, $\alpha$, $\beta$, $K$)
            \State $\matS{t} \leftarrow \Lrwr{t}$  \algorithmiccomment{\textbf{\scriptsize Spatial augmenter}}
            \State $\T{t} \leftarrow \matS{t}\Xtilde{t-1}$ \algorithmiccomment{\textbf{\scriptsize Temporal augmenter}}
            \State  $\X{t} \leftarrow (1 - \gamma)\matS{t} + \gamma\T{t}$ where $\gamma = \beta/(\alpha + \beta)$
            \State $\Xtilde{t} \leftarrow$ filter entries of $\X{t}$ if their weights are $< \epsilon$
            \State normalize $\Xtilde{t}$ column-wise
        \State
        \Return $\Xtilde{t}^{\top}$
        \Function{Power-Iteration}{$\Atilde{t}$, $\alpha$, $\beta$, $K$}\label{alg:method:power:start}
            \State set $c \leftarrow 1 - \alpha - \beta$ and $\mat{M}_{t}^{(0)} \leftarrow \I{n}$
            \For{\text{$k \leftarrow 1$ to $K$}}
                \State $\mat{M}_{t}^{(k)} \leftarrow \I{n} + c\Atilde{t}^{\top}\mat{M}_{t}^{(k-1)}$
            \EndFor
            \State $\Lrwr{t} \leftarrow (1-c)\mat{M}_{t}^{(K)}$ where $\mat{M}_{t}^{(K)}
            \approxeq \matL{t}^{-1}$
            \State normalize $\Lrwr{t}$ column-wise and \Return $\Lrwr{t}$
        \EndFunction\label{alg:method:power:end}
    \end{algorithmic}
\end{algorithm}

\def\arraystretch{1}
\setlength{\tabcolsep}{4.8pt}
\begin{table}[t!]
    \captionsetup{font=small}
    \caption{
        Summary of datasets. $n$ and $m$ are the total numbers of nodes and edges, resp.
        $T$ and $L$ are the numbers of time steps and labels, resp.
        $\bar{n}_t$ and $\bar{m}_t$ are the average numbers of activated nodes and edges over time, resp.
        $\bar{C}_{t} = \bar{m}_{t}/\bar{n}_t$.
        The first 3 data are used for link prediction, and the others are for node classification.
    }
    \label{tab:data}
    \small
    \begin{tabular}{lrrrrrr}
        \hline
        \toprule
        \textbf{Datasets} & \multicolumn{1}{r}{$n$} & \multicolumn{1}{r}{$m$} & \multicolumn{1}{r}{$T$} & \multicolumn{1}{r}{$L$} & \multicolumn{1}{r}{$\lfloor\bar{n}_{t}\rfloor$} & \multicolumn{1}{r}{$\bar{C}_{t}$} \\
        \midrule
        \textbf{BitcoinAlpha} & 3,783 & 31,748 & 138 & 2 & 105 & 2.2 \\
        \textbf{WikiElec} & 7,125 & 212,854 & 100 & 2 & 354 & 6.0 \\
        \textbf{RedditBody} & 35,776 & 484,460 & 88 & 2 & 2,465 & 2.2 \\
        \midrule
        \textbf{Brain} & 5,000 & 1,955,488 & 12 & 10 & 5,000 & 32.6 \\
        \textbf{DBLP-3} & 4,257 & 23,540 & 10 & 3 & 782 & 3.0 \\
        \textbf{DBLP-5} & 6,606 & 42,815 & 10 & 5 & 1,212 & 3.5 \\
        \textbf{Reddit} & 8,291 & 264,050 & 10 & 4 & 2,071 & 12.8 \\
        \bottomrule
        \hline
    \end{tabular}
\end{table}

\textbf{Discussion.}
Theorem~\ref{theorem:comp} implies that \method is faster than $O(n^3)$, and uses space less than $O(n^2)$ for storing $\Xtilde{t}$ in most real dynamic graphs.
Nevertheless, its time complexity can reach $O(n^2)$ for a graph such as the Brain dataset; thus, for larger graphs, its scalability can be limited.
However, \method is based on matrix operations which are easy-to-accelerate using GPUs, and other diffusion methods such as GDC lie at the same complexity.
Furthermore, there are extensive works of efficient RWR computations~\cite{andersen2006local,DBLP:conf/sigmod/JungPSK17,DBLP:conf/sigmod/ShinJSK15,wang2017fora,hou2021massively} and accelerated multiplications of sparse matrices~\cite{srivastava2020matraptor}, which can make \method scalable.
In this work, we focus on effectively augmenting a dynamic graph, and leave further computational optimization on the augmentation as future work.

\section{Experiment}
In this section, we evaluate \method to show its effectiveness for the augmentation problem for dynamic graphs.

\subsection{Experimental Setting}

\def\arraystretch{0.9}
\setlength{\tabcolsep}{2.4pt}
\begin{table*}[t!]
\begin{threeparttable}[t]
    \caption{
        \label{tab:link_auc}
        Temporal link prediction accuracy (AUC) where \textsc{None} is a result without augmentation, and {\color{red}$\blacktriangle$} (or {\color{blue}$\blacktriangledown$}) indicates improvement (or degradation) compared to None.
        \method shows consistent improvement across most models and datasets.
    }
    \begin{tabular}{@{}cccccccccc@{}}
        \hline
        \toprule
        \multirow{2}[2]{*}{\textbf{AUC}} & \multicolumn{3}{c}{\textbf{BitcoinAlpha}} & \multicolumn{3}{c}{\textbf{WikiElec}} & \multicolumn{3}{c}{\textbf{RedditBody}} \\
        \cmidrule(lr){2-4} \cmidrule(lr){5-7} \cmidrule(lr){8-10}
        & \textbf{GCN} & \textbf{GCRN} & \textbf{EGCN} & \textbf{GCN} & \textbf{GCRN} & \textbf{EGCN} & \textbf{GCN} & \textbf{GCRN} & \textbf{EGCN} \\
        \midrule
        \textsc{None} & \lwt57.3±1.6\lwt & \lwt80.3±6.0\lwt & \lwt58.8±1.1\lwt & \lwt59.9±0.9\lwt & \lwt72.1±2.4\lwt & \lwt66.9±3.7\lwt & \lwt77.6±0.4\lwt & \lwt88.9±0.3\lwt & \lwt77.6±0.2\lwt \\
        \midrule
        \textsc{DropEdge} & \lbt56.3±1.0\lwt & \lbt73.9±2.2\lwt & \lbt57.4±0.9\lwt & \lbt50.1±1.0\lwt & \lbt56.0±9.3\lwt & \lbt47.9±6.4\lwt & \lbt73.0±0.4\lwt & \lbt77.0±1.7\lwt & \lbt71.9±0.7\lwt \\
        \textsc{GDC} & \lrt57.5±1.6\lwt & \lbt77.3±6.5\lwt & \lbt57.4±1.2\lwt & \lrt62.8±0.8\lwt & \lbt67.9±1.0\lwt & \lbt63.1±0.7\lwt & \lbt74.6±0.0\lwt & \lbt86.4±0.3\lwt & \lbt73.8±0.3\lwt \\
        \textsc{Merge} & \lrt66.8±2.6\lwt & \lrt93.1±0.4\lwt & \lrt61.0±9.2\lwt & \lrt60.6±1.7\lwt & \lbt68.4±3.2\lwt & \lbt60.7±1.3\lwt & \lbt69.7±0.7\lwt & \lrt{89.8±0.5}\lwt & \lrt80.3±0.5\lwt \\
        \midrule
        \method & \lrt\textbf{76.0±1.3}\lwt & \lrt\textbf{94.6±0.8}\lwt & \lrt\textbf{77.2±1.4}\lwt & \lrt\textbf{69.0±1.2}\lwt & \lrt\textbf{73.4±2.2}\lwt & \lrt\textbf{69.1±0.3}\lwt & \lrt\textbf{80.8±0.6}\lwt & \lrt\textbf{90.2±0.4}\lwt & \lrt\textbf{82.0±0.1}\lwt \\
        \bottomrule
        \hline
    \end{tabular}
\end{threeparttable}
\end{table*}

\textbf{Datasets.}
Table~\ref{tab:data} summarizes $7$ public datasets used in this work.
BitcoinAlpha is a social network between bitcoin users~\cite{kumar2016edge,kumar2018rev2}.
WikiElec is a voting network for Wikipedia adminship elections~\cite{leskovec2010predicting}.
RedditBody is a hyperlink network of connections between two subreddits~\cite{kumar2018community}.
For node classification, we use the following datasets evaluated in~\cite{xu2019spatio}.
Brain is a network of brain tissues where edges indicate their connectivities.
DBLP-3 and DBLP-5 are co-authorship networks extracted from DBLP.
Reddit is a post network where two posts were connected if they contain similar keywords.

\textbf{Baseline augmentation methods.}
We compare \method to the following baselines.
\textsc{None} indicates the result of a model without any augmentation.
\textsc{DropEdge} is a drop-based method randomly removing edges at each epoch.
\textsc{GDC} is a graph diffusion-based method where we use PPR for this as our approach is based on random walks.
\textsc{Merge} is a simple baseline merging adjacency matrices from time $1$ to $t$ when training a model at time $t$.
We apply \textsc{DropEdge} and \textsc{GDC} to each snapshot since they are designed for a static graph.

\textbf{Baseline GNNs.}
We use GCN~\cite{kipf2017semi}, GCRN~\cite{seo2018structured} and EvolveGCN~\cite{ParejaDCMSKKSL20}, abbreviated to EGCN, for performing dynamic graph tasks.
We naively apply a static GCN to each graph snapshot for verifying how temporal information is informative.
We choose GCRN and EvolveGCN, lightweight and popular dynamic GNN models showing decent performance, to observe practical gains from augmentation.
We adopt GCN layers for GCRN's graph convolution.
We use the implementation of \cite{rozemberczki2021pytorch} for GCRN and EGCN.
Note that any GNN models following Problem~\ref{prob:aug} can utilize \method because our approach is model-agnostic.

\textbf{Training details.}
For each dataset, we tune the hyperparameters of all models on the original graph (marked as \textsc{None}) and augmented graphs separately through a combination of grid and random search on a validation set, and report test accuracy at the best validation epoch.
For \method, we fix $K$ to $100$, search for $\epsilon$ in $[0.0001, 0.01]$, and tune $\alpha$ and $\beta$ in $(0, 1)$ s.t. $0\!<\!\alpha + \beta\!<\!1$.
We use the Adam optimizer with weight decay $10^{-4}$, and the learning rate is tuned in $[0.01, 0.05]$ with decay factor $0.999$.
The dropout ratio is searched in $[0, 0.5]$.
We repeat each experiment $5$ times with different random seeds, and report the average and standard deviation of test values.
We use PyTorch and DGL~\cite{wang2019dgl} to implement all methods.
All experiments were done at workstations with Intel Xeon 4215R and RTX 3090.
Details about the experimental setting are provided in Appendix~\ref{sec:appendix:setting}.

\subsection{Temporal Link Prediction Task}
This aims to predict whether an edge exists or not at time $t+1$ using the information up to time $t$.
As a standard setting~\cite{ParejaDCMSKKSL20}, we follow a chronological split with ratios of training (70\%), validation (10\%), and test (20\%) sets.
We sample the same amount of negative samples (edges) to positive samples (edges) for each time, and use AUC as a representative measure.
We set the number of epochs to $200$ with early stopping of patience $50$.

As shown in Table~\ref{tab:link_auc}, \method consistently improves the performance of dynamic GNN models such as GCRN and EGCN compared to \textsc{None} (i.e., without augmentation) while static augmentations of \textsc{DropEdge} and \textsc{GDC} do not.
\method also outperforms the static methods on all models and datasets.
This indicates it is not beneficial to only spatially augment the graphs for this task.
\method even improves static GCN, which is competitive with EGCN, implying that effectively and temporally augmented data can even make static GNNs learn dynamic graphs well.
In addition, \textsc{Merge} also improves the accuracy of the tested models on many datasets.
This confirms the need to utilize temporal information when it comes to dynamic graph augmentation in this task.
However, \textsc{Merge} performs worse than \method in most cases because \method can effectively augment both spatial and temporal localities at once while \textsc{Merge} does not have a mechanism to enhance such localities.

\def\arraystretch{1.1}
\setlength{\tabcolsep}{0.8pt}
\begin{table*}[t!]
\small
\begin{threeparttable}[t]
    \caption{
        \label{tab:node_f1}
        Node classification accuracy (Macro F1-score) where \textsc{None} is a result without augmentation, and {\color{red}$\blacktriangle$} (or {\color{blue}$\blacktriangledown$}) indicates improvement (or degradation) compared to None.
        \method shows consistent improvement across most models and datasets.
    }
    \begin{tabular}{@{}ccccccccccccc@{}}
        \hline
        \toprule
        \multirow{2}[2]{*}{\textbf{Macro F1}} & \multicolumn{3}{c}{\wt\textbf{Brain}} & \multicolumn{3}{c}{\wt\textbf{Reddit}} & \multicolumn{3}{c}{\wt\textbf{DBLP-3}} & \multicolumn{3}{c}{\wt\textbf{DBLP-5}} \\
        \cmidrule(lr){2-4} \cmidrule(lr){5-7} \cmidrule(lr){8-10} \cmidrule(lr){11-13}
        & \wt\textbf{GCN} & \wt\textbf{GCRN} & \wt\textbf{EGCN} & \wt\textbf{GCN} & \wt\textbf{GCRN} & \wt\textbf{EGCN} & \wt\textbf{GCN} & \wt\textbf{GCRN} & \wt\textbf{EGCN} & \wt\textbf{GCN} & \wt\textbf{GCRN} & \wt\textbf{EGCN} \\
        \midrule
        \textsc{None} & \wt44.7±0.8 & \wt66.8±1.0 & \wt43.4±0.7 & \wt18.2±2.9 & \wt40.4±1.6 & \wt18.6±2.3 & \wt53.4±2.6 & \wt83.1±0.6 & \wt51.3±2.7 & \wt69.6±0.9 & \wt75.4±0.7 & \wt68.5±0.6 \\
        \midrule
        \textsc{DropEdge} & \bt35.2±1.7 & \rt67.8±0.6 & \bt39.7±1.8 & \rt\textbf{19.4±0.8} & \bt40.3±1.4 & \bt{18.0±2.7} & \rt55.8±1.9 & \rt84.3±0.6 & \rt52.4±1.7 & \rt70.5±0.5 & \rt75.6±0.7 & \bt68.0±0.7 \\
        \textsc{GDC} & \rt63.2±1.2 & \rt88.0±1.5 & \rt67.3±1.3 & \bt17.5±2.3 & \rt41.0±1.6 & \bt18.5±2.8 & \rt53.4±2.1 & \rt84.7±0.5 & \rt52.8±2.2 & \rt70.0±0.7 & \rt75.5±1.2 & \rt69.1±1.0 \\
        \textsc{Merge} & \bt34.4±3.4 & \bt63.2±1.6 & \rt53.0±0.9 & \rt19.3±3.0 & \bt39.6±0.8 & \rt 20.4±3.0 & \rt54.9±3.1 & \bt83.0±1.4 & \rt53.3±1.2 & \rt{70.8±0.4} & \bt74.5±0.8 & \rt69.7±1.6 \\
        \midrule
        \textbf{\method} & \rt\textbf{68.7±1.2} & \rt\textbf{91.3±1.0} & \rt\textbf{72.0±0.6} & \rt{18.4±3.0} & \rt\textbf{41.5±1.5} & \rt\textbf{21.9±1.6} & \rt\textbf{57.5±2.2} & \rt\textbf{84.9±1.6} & \rt\textbf{56.4±1.8} & \rt{\textbf{71.1±0.6}} & \rt\textbf{77.9±0.4} & \rt\textbf{70.1±1.0} \\
        \bottomrule
        \hline
    \end{tabular}
\end{threeparttable}
\end{table*}

\subsection{Node Classification Task}
This is to classify a label of each node where a graph and features change over time.
Following~\cite{xu2019spatio}, we split all nodes into training, validation, and test sets by the {7:1:2} ratio.
{We feed node embeddings $\matH{T}$ of each model forward to a softmax classifier}, and use Macro F1-score because labels are imbalanced in each dataset.
We set the number of epochs to $1,000$ with early stopping of patience $100$.

Table~\ref{tab:node_f1} shows \method consistently improves the accuracies of GNNs on most datasets.
Especially, \method significantly enhances the accuracies on the Brain dataset as another diffusion method GDC does, but \method shows better accuracy than GDC, implying it is effective to augment a temporal locality for the performance.
For the other datasets, \method slightly improves each model, but it overall performs better than other augmentations.
Note GCN and EGCN are worse than a random classifier of $1/L$ score (0.25 for $L\!=\!4$) in the Reddit where $L$ is the number of labels, and all tested augmentations fail to beat the score, implying even these augmentations could not boost a poor model in this task.

\subsection{Effect of Hyperparameters}
We analyze the effects of temporal decay ratio $\gamma$ and filtering threshold $\epsilon$ that mainly affect \method's results.
We fix the number $K$ of iterations to $100$ for the power iteration, which leads to sufficiently accurate results for $\Lrwr{t}$.

\begin{figure}[t]
    \subfigure{
        \label{fig:alpha:effect_gamma}
        \includegraphics[width=1.0\linewidth]{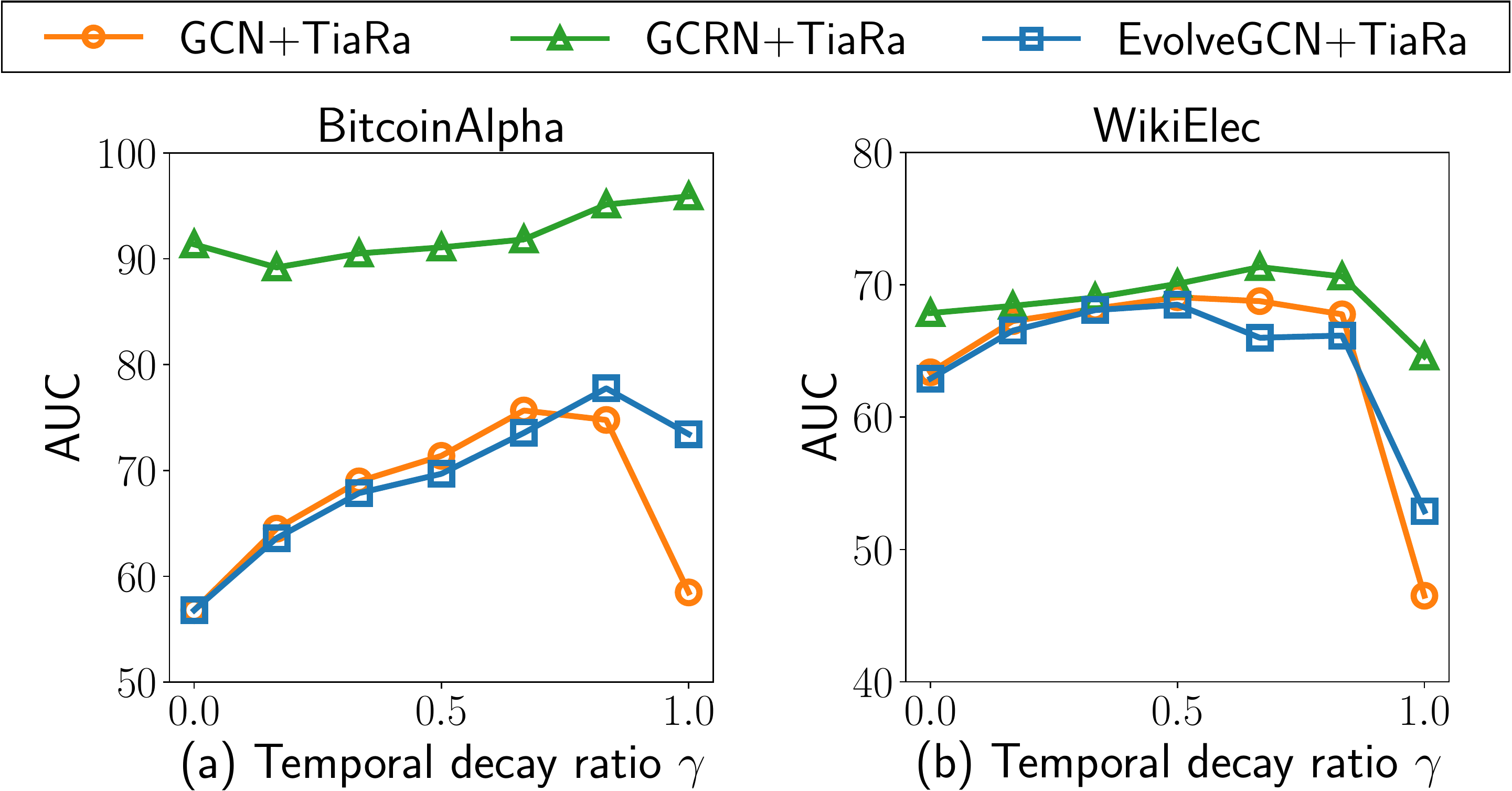}
    }
    \caption{
        \label{fig:effect_gamma}
        {
        Effect of the temporal decay ratio $\gamma$.
        }
    }
\end{figure}

\textbf{Effect of the temporal decay ratio $\gamma$.}
As \method's hyperparameters, $\alpha$ and $\beta$ should be analyzed, but our preliminary experiments showed that patterns vary by models and datasets in the changes of $\alpha$ and $\beta$.
Instead, we narrow our focus to $\gamma$ in Equation~\eqref{eq:trwr_diff_rec} where $\gamma = \beta/(\alpha + \beta)$.
For this experiment, we vary $\gamma$ from $10^{-5}$ to $1\!-\!10^{-5}$ by tweaking $\alpha$ and $\beta$ s.t. $\alpha + \beta$ is fixed to $0.6$.
Figure~\ref{fig:effect_gamma} shows that too small or large values of $\gamma$ can degrade link prediction accuracy except GCRN with \method in BitcoinAlpha.
This implies that it is important to properly mix spatial and temporal information about the performance, which is controlled by $\gamma$.

\begin{figure}[t]
    \subfigure{
        \label{fig:alpha:effect_eps}
        \includegraphics[width=1.0\linewidth]{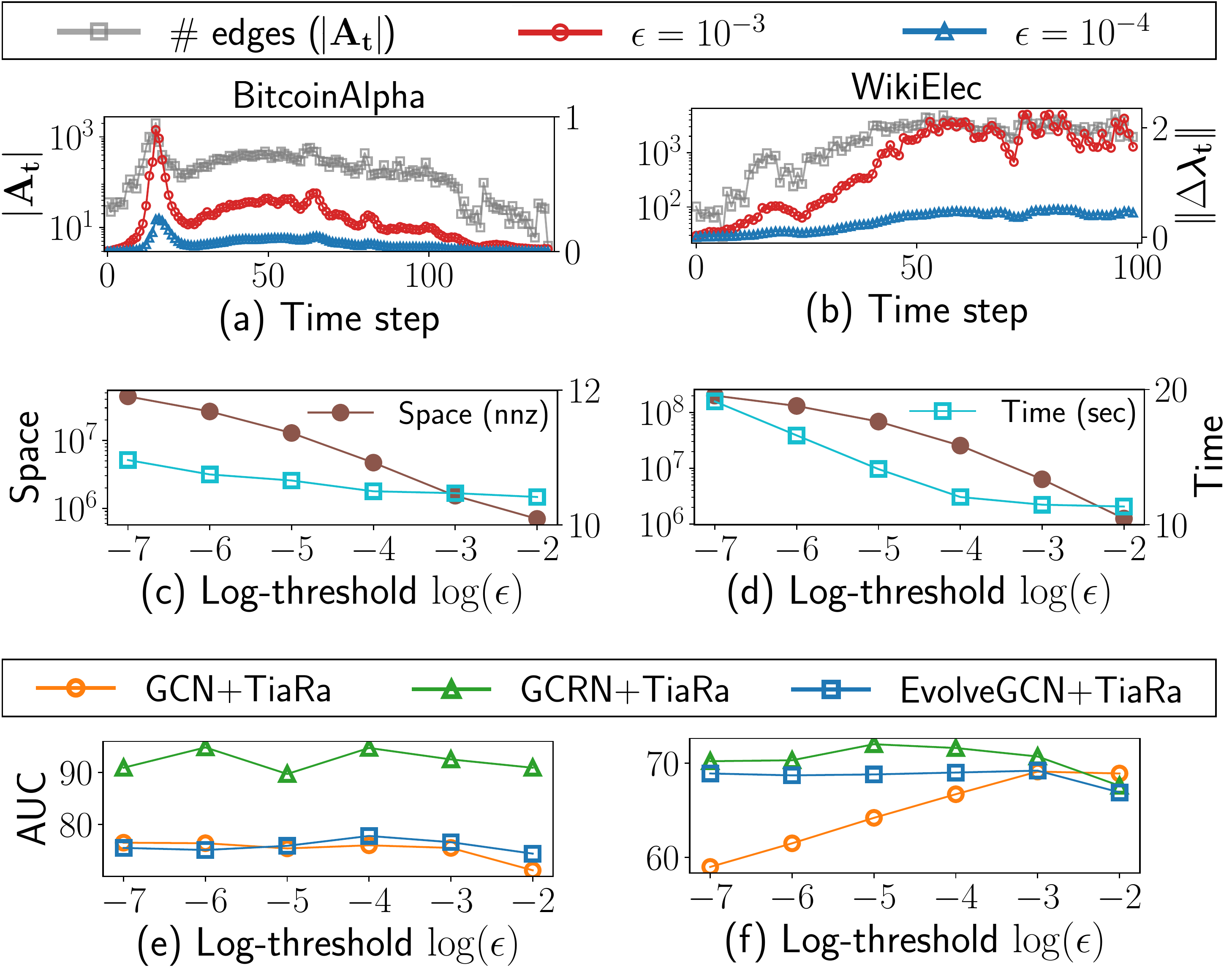}
    }
    \caption{
        \label{fig:effect_eps}
        Effect of the filtering threshold $\epsilon$.
    }
\end{figure}

\textbf{Effect of the filtering threshold $\epsilon$.}
Figure~\ref{fig:effect_eps} shows the effects of $\epsilon$ in terms of approximate error, time, space, and accuracy of link prediction in BitcoinAlpha and WikiElec.
We fix $\alpha$ and $\beta$ to 0.25, and vary $\epsilon$ from $10^{-7}$ to $10^{-2}$ for this experiment.

We measure the approximate error $\lVert\Delta\lmd{t}\rVert\!=\!\lVert \lmd{t}\!-\!\tilde{\lmd{t}}\rVert_{2}$ of eigenvalues where $\lmd{t}$ and $\tilde{\lmd{t}}$ are vectors of eigenvalues of $\X{t}$ (i.e., $\epsilon\!=\!0$) and $\Xtilde{t}$, respectively, as similarly analyzed in~\cite{klicpera2019diffusion}.
The right y-axis of Figures~\ref{fig:effect_eps}(a) and (b) is the error, and the left y-axis is the number $|\A{t}|$ of edges in $\A{t}$.
As time $t$ increases, the errors (red and blue lines) remain small, and do not explode, implying errors incurred by repeated sparsifications are not excessively accumulated over time.
Rather, the errors tend to be proportional to $|\A{t}|$ at each time.

Figures~\ref{fig:effect_eps} (c) and (d) show the space measured by $\sum_{t}|\Xtilde{t}|$ (left y-axis) and the augmentation time (right y-axis) of \method by $\epsilon$ between $10^{-7}$ and $10^{-2}$.
As the strength of sparsification increases (i.e., $\epsilon$ becomes larger), the produced non-zeros and the augmentation time decrease.
On the other hand, most of the accuracies remain similar except  $\epsilon\!=\!10^{-2}$ as shown in Figures~\ref{fig:effect_eps}(e) and (f).
Note it is not effective to truncate too many entries (e.g., $\epsilon\!=\!10^{-2}$), or too dense $\Xtilde{t}$ can worse the performance as GCN in WikiElec.
Thus, the sparsification with proper $\epsilon$ such as $10^{-3}$ or $10^{-4}$ provides a good trade-off between error, time, space, and accuracy.

\section{Conclusion}
In this work, we propose \method, a novel and model-agnostic diffusion method for augmenting a dynamic graph with the purpose of improvements in dynamic GNN models.
We first extend Random Walk with Restart (RWR) to Time-aware RWR so that it produces spatially and temporally localized scores.
We then formulate time-aware random walk diffusion matrices, and analyze how our diffusion approach augments both spatial and temporal localities in the dynamic graph.
As graph diffusions lead to dense matrices, we further employ approximate techniques such as power iteration and sparsification, and analyze how they are effective for achieving a good trade-off between error, time, space, and predictive accuracy.
Our experiments on various real-world dynamic graphs show that \method aids GNN models in providing better performance of temporal link prediction and node classification tasks.

{
\bibliography{reference}
}

\newpage
\appendix
\setcounter{secnumdepth}{1}
\setlength{\extrarowheight}{0em}
\def\arraystretch{0.95}
\setlength{\tabcolsep}{2.7pt}
\begin{table}[!t]
    \vspace{-3mm}
	\centering
	\small
	\caption{Table of symbols.}
	\label{tab:symbols}
	\begin{tabular}{cl}
	    \hline
		\toprule
		\textbf{Symbol} & \textbf{Definition} \\
        \midrule
        $\DTDG = \{\G{1},\cdots,\G{T}\}$ & {\small discrete-time dynamic graph} \\
        $\G{t} = (\V{}, \E{t}, \F{t})$ & {\small graph snapshot at time $t$ of $\DTDG$} \\
        $\E{t}$ & {\small set of edges at time $t$} \\
        $\V{}$ & {\small set of nodes where $n = |\V{}|$} \\
        $\V{t} \subset \V{}$ & {\small set of activated nodes where $n_t = |\V{t}|$} \\
        $\F{t} \in \mathbb{R}^{n \times d}$ & {\small initial node feature matrix at time $t$} \\
         & {\small where $d$ is the number of features} \\
        {$\A{t} \in \mathbb{R}^{n \times n}$} & {\small adjacency matrix of $\G{t}$} \\
        {$\Atilde{t} \in \mathbb{R}^{n \times n}$} & {\small normalized adjacency matrix of $\G{t}$} \\
        {$\mathcal{F}_{\Theta}(\cdot)$} & {\small GNN for discrete-time dynamic graphs} \\
        {$\I{n} \in \mathbb{R}^{n \times n}$} & {\small identity matrix} \\
        {$\alpha$} & {\small restart probability} \\
        {$\beta$} & {\small time travel probability} \\
        {$K$} & {\small number of power iterations} \\
        {$\epsilon$} & {\small filtering threshold} \\
        {$\gamma$} & {\small temporal decay ratio $\gamma=\beta/(\alpha + \beta)$} \\
        {$\matL{t} \in \mathbb{R}^{n \times n}$} & {\small  $\L{t} = \I{n} - (1 - \alpha - \beta)\Atilde{t}^{\top}$} \\
        {$\Lrwr{t} \in \mathbb{R}^{n \times n}$} & {\small RWR diffusion kernel on the graph $\G{t}$ } \\
        {} & {where $\Lrwr{t} = (\alpha + \beta)\L{t}^{-1}$}\\
        {$\matS{t} \in \mathbb{R}^{n \times n}$} & {\small spatial diffusion matrix} \\
        {$\T{t} \in \mathbb{R}^{n \times n}$} & {\small temporal diffusion matrix} \\
        {$\X{t} \in \mathbb{R}^{n \times n}$} & {\small time-aware random walk diffusion matrix} \\
		\bottomrule
		\hline
	\end{tabular}
\end{table}

\section{Proofs}
\label{sec:appendix:proof}

\begin{lemma}
\label{lemma:transition}
For every time step $t$, $\X{t}$ is column stochastic.
\end{lemma}

\begin{proof}
As a base case, $\X{0} = \I{n}$ for $t=0$, which is trivially column stochastic.
Assume $\X{t-1}$ is column stochastic, i.e., $\vect{1}^{\top}\X{t-1} = \vect{1}^{\top}$
where $\vect{1}\in \mathbb{R}^{n}$ is a column vector of ones.
Then, $\vect{1}^{\top}\X{t}$ is written as follows:
\begin{align*}
    \begin{split}
        \vect{1}^{\top}\X{t}
        &= (1-\gamma) \vect{1}^{\top}\Lrwr{t}\I{n} + \gamma\vect{1}^{\top}\Lrwr{t}\X{t-1} \\
        &= (1-\gamma)\vect{1}^{\top}\I{n} + \gamma\vect{1}^{\top}\X{t-1} \\
        &= (1-\gamma)\vect{1}^{\top} + \gamma\vect{1}^{\top} = \vect{1}^{\top}
    \end{split}
\end{align*}
where $\vect{1}^{\top}\Lrwr{t} = \vect{1}^{\top}$.
Therefore, the claim holds for every $t \geq 0$ by mathematical induction.
\end{proof}

\noindent\textbf{Proof of Theorem~\ref{theorem:trwr_diff_closed}}. We prove Theorem~\ref{theorem:trwr_diff_closed} as follows:

\begin{proof}
We begin the derivation from the following equation:
\begin{equation}
    \X{t+1} = (1-\gamma)\left(\Lrwr{t+1}\I{n}\right) + \gamma\left(\Lrwr{t+1}\X{t}\right) \tag{\ref{eq:trwr_diff_rec}}
\end{equation}
As a base case, $\X{1}$ is ($t=0$ in the above) as follows:
\begin{equation*}
    \X{1} = (1-\gamma)\left(\Lrwr{1}\I{n}\right) + \gamma\left(\Lrwr{1}\X{0}\right) = \Lrwr{1}
\end{equation*}
where $\X{0}=\I{n}$. This trivially holds Equation~\eqref{eq:trwr_diff_closed}.
Let's assume that Equation~\eqref{eq:trwr_diff_closed} holds at $t$.
Then, by substituting $\X{t}$ of Equation~\eqref{eq:trwr_diff_closed} into Equation~\eqref{eq:trwr_diff_rec}, $\X{t+1}$ is:
\begin{align*}
    \begin{split}
        \X{t+1} &= (1\!-\!\gamma)\left(\Lrwr{t+1} + \sum_{i = 0}^{t-2} \gamma^{i+1} \Lrwr{t+1 \leftlsquigarrow t-i} \right)\!+\!\gamma^{t}\Lrwr{t+1\leftlsquigarrow1} \\
        &= (1-\gamma)\left(\sum_{i = 0}^{t-1} \gamma^{i} \Lrwr{t+1 \leftlsquigarrow t+1-i} \right)\!+\!\gamma^{t}\Lrwr{t+1\leftlsquigarrow1}
    \end{split}
\end{align*}

Suppose $k=t+1$; then, $\X{k}$ is represented as follows:
\begin{equation*}
    \X{k} = (1-\gamma)\left(\sum_{i = 0}^{k-2} \gamma^{i} \Lrwr{k \leftlsquigarrow k-i} \right)\!+\!\gamma^{k-1}\Lrwr{k\leftlsquigarrow1}
\end{equation*}

Note that the equation of $\X{k}$ has the same form of that of $\X{t}$ in Equation~\eqref{eq:trwr_diff_rec}.
This indicates Equation~\eqref{eq:trwr_diff_rec} also holds at $k=t+1$.
Thus, the claim holds for every $t \geq 1$ by mathematical induction.
\end{proof}

\begin{lemma}
\label{lemma:error}
Suppose $c = 1 - \alpha - \beta$ and $0 < c < 1$.
The approximate error of the following power iteration is bounded by $c^{k}$, and converges to $0$ as $k \rightarrow \infty$:
\begin{align*}
    \mat{M}_{t}^{(k)} = \I{n} + c\Atilde{t}^{\top}\mat{M}_{t}^{(k-1)}
\end{align*}
where $\Atilde{t}$ is column stochastic, and $\mat{M}_{t}^{(0)} = \I{n}$.
\end{lemma}

\begin{proof}
Let $\mat{M}_{t}^{*}$ be the stationary matrix of the equation. Then, the iteration is represented as $\mat{M}_{t}^{*} = \I{n} + c\Atilde{t}^{\top}\mat{M}_{t}^{*}$, implying $\mat{M}_{t}^{*}= \matL{t}^{-1}$.
Then, the error $\lVert \mat{M}_{t}^{*} -  \mat{M}_{t}^{(k)}\rVert_{1}$ is represented as follows:
\begin{align*}
    \lVert \mat{M}_{t}^{*} -  \mat{M}_{t}^{(k)}\rVert_{1}
    &= \lVert c\Atilde{t}^{\top}\mat{M}_{t}^{*} -  c\Atilde{t}^{\top}\mat{M}_{t}^{(k-1)}\rVert_{1} \\
    &\leq c \lVert\Atilde{t}^{\top}\rVert_{1} \lVert \mat{M}_{t}^{*} - \mat{M}_{t}^{(k-1)} \rVert_{1} \\
    &\cdots \\
    &\leq c^{k} \lVert \mat{M}_{t}^{*} - \mat{M}_{t}^{(0)} \rVert_{1} = c^{k} \lVert \mat{M}_{t}^{*} - \I{n} \rVert_{1} \\
    &\leq c^{k} \left(\frac{c}{1-c}\right)
\end{align*}
where $\lVert\cdot\rVert_{1}$ is L1 norm of a matrix.
Note $\lVert(1-c)\matL{t}^{-1}\rVert_{1} = 1$ since $(1-c)\matL{t}^{-1}$ is stochastic, and $\mat{M}_{t}^{*} = \matL{t}^{-1}$.

Hence, $\lVert\mat{M}_{t}^{*}\rVert_{1}\!=\!(1-c)^{-1}$.
Then, each column sum of $\mat{M}_{t}^{*} - \I{n}$ is $(1-c)^{-1}-1$; thus, $\lVert \mat{M}_{t}^{*} - \I{n} \rVert_{1} = c/(1\!-\!c)$.
Since $c/(1-c)$ is a constant, the error converges to $0$ as $k\!\rightarrow\!\infty$.
\end{proof}

\noindent\textbf{Proof of Theorem~\ref{theorem:comp}}. We prove Theorem~\ref{theorem:comp} as follows:

\begin{proof}
According to Lemma~\ref{lemma:nnz}, $\text{nnz}(\Xtilde{t}) = O(n/\epsilon)$.
Thus, the space complexity of $\Xtilde{t}$ is $O(n/\epsilon)$ if we use a sparse matrix format such as CSR or CSC.
The time complexity is proved in Lemma~\ref{lemma:time}.
\end{proof}

\begin{lemma}
\label{lemma:nnz}
The sparsification truncates small entries of $\X{t}$ if their weights are $<\epsilon$.
Then, $\Xtilde{t}$ has $n/\epsilon$ non-zeros at most, i.e., $\text{nnz}(\Xtilde{t}) = O(n/\epsilon)$ after the sparsification.
\end{lemma}

\begin{proof}
Let $\vect{\tilde{x}}_{t,s}$ denote the $s$-th column of $\Xtilde{t}$.
Then, the claim is represented as $\text{nnz}(\vect{\tilde{x}}_{t,s}) \leq  \frac{1}{\epsilon}$, and proved by contradiction.
Let's say the claim is false, i.e., $\text{nnz}(\vect{\tilde{x}}_{t,s}) >  \frac{1}{\epsilon}$.
Assume $\vect{\tilde{x}}_{t,s}$ has $\frac{1}{\epsilon} + 1$ non-zeros after the sparsification.
This implies their weights were $\geq \epsilon$, and the sum of their weights was $\geq (\frac{1}{\epsilon} + 1)\epsilon = 1 + \epsilon$ before the sparsification.
For $0 < \epsilon < 1$, this contradicts to the fact that $\X{t}$ is column stochastic, i.e., each column sum of $\X{t}$ is $1$.
Therefore, $\text{nnz}(\vect{\tilde{x}}_{t,s}) \leq  \frac{1}{\epsilon}$ after the sparsification, and it holds for every column.
Thus, $\text{nnz}(\Xtilde{t}) \leq n/\epsilon = O(n/\epsilon)$.
\end{proof}

\begin{lemma}
\label{lemma:time}
For each time step $t$, Algorithm 1 of \method takes $O(n_tn/\epsilon + n_t^2K)$ time in expectation.
\end{lemma}

\begin{proof}
Table~\ref{tab:time_comp_each} summarizes the time complexity of each step of \method.
Line 1 normalizes the self-looped adjacency matrix $\A{t}$.
Let $m_t$ be the number of edges in $\A{t}$ without self-loops, and $n$ be the number of all nodes.
Then, it takes $O(m_t + n)$ to obtain $\mat{D}_{t}^{-1}$ and multiply it and $\A{t}$.

Line 2 performs the power-iteration to compute $\Lrwr{t}$.
Let $\V{t}$ denote the set of activated nodes, and $n_t=|\V{t}|$ at time $t$. Then, $\Atilde{t}^{\top}$ is reordered by activated and non-activated nodes as follows:
\begin{equation*}
    \Atilde{t}^{\top} = \begin{bmatrix}
        \Atilde{\V{t}}^{\top} & \mat{O} \\
        \mat{O}^{\top} & \I{n-n_t}
    \end{bmatrix}
\end{equation*}
where $\mat{O}\!\in\!\mathbb{R}^{n_t \times n-n_t}$ is a zero matrix, and $\Atilde{\V{t}}\!\in\!\mathbb{R}^{n_t \times n_t}$ is the normalized adjacency matrix for the activated nodes in $\V{t}$.
Then, the $k$-th iteration is represented as follows:
{\small
\begin{align*}
    \mat{M}_{t}^{(k)} &= \I{n} + c\Atilde{t}^{\top}\mat{M}_{t}^{(k-1)} \\
    \mat{M}_{t}^{(k)} &=
    \begin{bmatrix}
        \I{n_t}  & \mat{O} \\
        \mat{O}^{\top} & \I{n-n_t}
    \end{bmatrix}
    + c \begin{bmatrix}
        \Atilde{\V{t}}^{\top} & \mat{O} \\
        \mat{O}^{\top} & \I{n-n_t}
    \end{bmatrix}
    \begin{bmatrix}
        \mat{M}_{t,11}^{(k-1)} & \mat{M}_{t,12}^{(k-1)} \\
        \mat{M}_{t,21}^{(k-1)} & \mat{M}_{t,22}^{(k-1)}
    \end{bmatrix} \\
    \mat{M}_{t}^{(k)} &=
    \begin{bmatrix}
        \I{n_t} + c\Atilde{\V{t}}^{\top}\mat{M}_{t,11}^{(k-1)} & c\Atilde{\V{t}}^{\top}\mat{M}_{t,12}^{(k-1)} \\
        c\mat{M}_{t,21}^{(k-1)} & \I{n-n_t} + c\mat{M}_{t,22}^{(k-1)}
    \end{bmatrix}
\end{align*}
}

\setlength{\tabcolsep}{13pt}
\begin{table}[t!]
    \vspace{-3mm}
    \caption{Time complexity of Algorithm~\ref{alg:method}}
    \label{tab:time_comp_each}
    \small
    \begin{tabular}{@{}cll@{}}
        \toprule
        \textbf{Line} & \textbf{Task} & \textbf{Time Complexity} \\
        \midrule
        1 & normalize $\A{t}$ & $O(m_t + n)$ \\
        2 & power-iterate $\Lrwr{t}$ & $O(n_t^2K + nK)$ \\
        4 & $\T{t} \leftarrow \matS{t}\Xtilde{t}$ & $O(nn_t/\epsilon)$ \\
        5 & $\X{t} \leftarrow (1 - \gamma)\matS{t} + \gamma\T{t}$ & $O(nn_t + n/\epsilon)$ \\
        6 & filter $\X{t}$ into $\Xtilde{t}$ & $O(nn_t + n/\epsilon)$ \\
        7 & normalize $\Xtilde{t}$ & $O(n/\epsilon)$ \\
        \midrule
        \multicolumn{1}{l}{\textbf{Total}} & & $O(n_tn/\epsilon + n_t^2K)$ \\
        \bottomrule
    \end{tabular}
\end{table}

Note that $\mat{M}_{t}^{(0)} = \I{n}$ at the beginning.
Thus, $\mat{M}_{t,12}^{(i)}$ and $\mat{M}_{t,21}^{(i)}$ remain as zeros for every iteration $i$ because they are off the diagonal.
Therefore, $\mat{M}_{t}^{(k)}$ is written as follows:
\begin{align*}
    \mat{M}_{t}^{(k)} =
    \begin{bmatrix}
        \I{n_t} + c\Atilde{\V{t}}^{\top}\mat{M}_{t,11}^{(k-1)} & \mat{O} \\
        \mat{O}^{\top} & \I{n-n_t} + c\mat{M}_{t,22}^{(k-1)}
    \end{bmatrix}
\end{align*}
where $\mat{M}_{t,22}^{(i)}$ remains diagonal as $\mat{M}_{t,22}^{(0)} = \I{n-n_t}$.
In the equation, $\Atilde{\V{t}}^{\top}\mat{M}_{t,11}^{(k-1)}$ takes $O(|\Atilde{\V{t}}^{\top}|n_t)$ time where $\Atilde{\V{t}}$ is sparse (Lemma~\ref{lemma:spmm}), and $|\Atilde{\V{t}}^{\top}| = m_t + n_t$ where $m_t$ and $n_t$ are the numbers of edges and self-loops, respectively.
Also, in real-world dynamic graphs, $m_t = Cn_t$ and $C$ is a constant.
The average of $m_t/n_t$ over time was between 2.2 and 32.6 in the datasets used in this work (see Table 1 of the submitted paper).
Thus, $\Atilde{\V{t}}^{\top}\mat{M}_{t,11}^{(k-1)}$ takes $O(n_t^{2})$ time on average.
Also, it takes $O(n-n_t)$ time to compute $\I{n-n_t} + c\mat{M}_{t,22}^{(k-1)}$; each iteration takes $O(n_t^2 + n)$; thus, $O(n_t^2K + nK)$ time is required for $K$ iterations.

Line 4 computes $\matS{t}\Xtilde{t}$ where $\matS{t} = \Lrwr{t}$. Note that the structure of $\Lrwr{t}$ is the same as that of $\mat{M}_{t}^{(K)}$, which is a block diagonal matrix as shown in the above.
Thus, $\matS{t}\Xtilde{t}$ is represented as follows:
\begin{align*}
    \small
    \T{t} = \matS{t}\Xtilde{t} =
    \begin{bmatrix}
        \matS{t,11} & \mat{O} \\
        \mat{O}^{\top} & \matS{t,22}
    \end{bmatrix}
    \begin{bmatrix}
        \Xtilde{t,11} & \Xtilde{t,12} \\
        \Xtilde{t,21} & \Xtilde{t,22}
    \end{bmatrix}
\end{align*}
where $\matS{t,22} = (1-c)\mat{M}_{t}^{(K)}$ converges to $\I{n-n_t}$ as $K$ increases, and it is normalized (line 15).
Then, the above is written as:
\begin{align*}
    \small
    \T{t} = \matS{t}\Xtilde{t}
    &=
    \begin{bmatrix}
        \matS{t,11} & \mat{O} \\
        \mat{O}^{\top} & \I{n-n_t}
    \end{bmatrix}
    \begin{bmatrix}
        \Xtilde{t,11} & \Xtilde{t,12} \\
        \Xtilde{t,21} & \Xtilde{t,22}
    \end{bmatrix} \\
    &=
    \begin{bmatrix}
        \matS{t,11}\Xtilde{t,11} & \matS{t,11}\Xtilde{t,12} \\
        \Xtilde{t,21} & \Xtilde{t,22}
    \end{bmatrix}
\end{align*}

Note $\Xtilde{t}$ is sparse, and has $O(n/\epsilon)$ non-zeros (Lemma~\ref{lemma:nnz}); thus, its sub-matrices have less non-zeros than that.
To exploit the sparsity, suppose we compute $(\Xtilde{t,11}^{\top}\matS{t,11}^{\top})^{\top}$ and $(\Xtilde{t,12}^{\top}\matS{t,11}^{\top})^{\top}$.
Then, they take $O(n_t n/\epsilon)$ time for sparse matrix multiplications (Lemma~\ref{lemma:spmm}).

Lines 5 and 6 compute $\X{t} \leftarrow (1 - \gamma)\matS{t} + \gamma\T{t}$, and filter $\X{t}$.
The worst number of non-zeros of $\matS{t}$ is $O(n_t^2 + n)$ for $\matS{t,11}$ and $\matS{t,22}$.
The worst number of non-zeros of $\T{t}$ is decided when $\T{t,11}$ and $\T{t,12}$ are fully dense; then, it is $O(nn_t + n/\epsilon)$.
Thus, $O(nn_t + n/\epsilon)$ time is required to add them and filter $\X{t}$ because it is expected to be proportional to $\text{nnz}(\X{t})$.
After filtering, $\Xtilde{t}$ has $O(n/\epsilon)$ non-zeros by Lemma~\ref{lemma:nnz}; it takes $O(n/\epsilon)$ time to normalize $\Xtilde{t}$.
\end{proof}

\begin{lemma}
\label{lemma:spmm}
Let $\mat{A}$ and $\mat{B}$ be $p \times q$ and $q \times r$ matrices, respectively, and $\mat{A}$ has $|\mat{A}|$ non-zero entries.
Then, the sparse matrix multiplication $\mat{C}=\mat{A}\mat{B}$ takes $O(|\mat{A}|r)$ time.
\end{lemma}

\begin{proof}
The $i$-th column of $\mat{C}$ is the result of a sparse matrix vector multiplication with $\mat{A}$ and the $i$-th column of $\mat{B}$, taking $O(|\mat{A}|)$ for a sparse matrix $\mat{A}$ in the CSR format.
Thus, for $r$ columns, it takes $O(|\mat{A}|r)$ time.
\end{proof}

\section{Discussions}
\label{sec:appendix:discussion}
We discuss several aspects about dynamic graph augmentation and our \method in this section.

\textbf{Insertion or deletion of edges.}
\method naturally supports the operations of edge insertion and deletion on discrete-time dynamic graphs (DTDG).
According to~\cite{skarding2021foundations}, a DTDG is an ordered sequence of graph snapshots, and it can represent \textit{temporal graphs} or \textit{evolving graphs}.
A temporal graph is a network of edges where they exist for a certain time range.
In DTDGs, a set of edges in graph snapshot $\G{t}$ only exists at time $t$.
Thus, it is more natural to think of a link as an event with a duration~\cite{skarding2021foundations}, rather than insertion or deletion.
By default, \method can augment a sequence of graph snapshots which represents such a temporal graph.

On the other hand, there are edge insertion and deletion in evolving graphs.
An evolving graph keeps the overall structure by continually adding or deleting edges.
Each snapshot also represents the adjacency matrix $\A{t}$ at each time $t$, but in evolving graphs, $\A{t}$ and $\A{t+1}$ are mostly similar while they are different in temporal graphs.
As the evolving graph supports the insertion and deletion on edges, \method also naturally supports them.
If an edge $(u, v)$ is inserted or deleted at time $t$, we simply mark $\A{t;uv}$ as $1$ (insertion) or $0$ (deletion), and then perform \method on the graph snapshot.

\textbf{Insertion or deletion of nodes.}
\method can support node insertion and deletion after a few modifications.
For the node insertion, suppose a node $u$ is inserted at time $t$, i.e., the node did not exist before time $t$.
Then, we assign the node $u$ to a position corresponding to the last index of $\Xtilde{t-1}$ as follows:
\begin{align*}
    \small
    \underbrace{\Xtilde{t-1}}_{\text{Before}}
    \xRightarrow{\text{$u$ is inserted}}
    \underbrace{
        \Xtilde{t-1} \leftarrow \begin{bmatrix}
            \Xtilde{t-1} & \vect{0} \\
            \vect{0}^{\top} & 1
        \end{bmatrix}
    }_{\text{After}}
\end{align*}
where $\vect{0}$ is an $n$-dimensional column vector of size $n$.
After then, \method performs its diffusion with new $\A{t}$ and $\Xtilde{t-1}$ in which the node $u$ is inserted.

For the node deletion, suppose a node $u$ is deleted at time $t$, i.e., the node will not exist from time $t$.
For this, we delete the $u$-th row and column of $\Xtilde{t-1}$, and normalize the modified $\Xtilde{t-1}$ column-wise to make it column stochastic.
After reordering node indices to match remaining nodes, \method performs its diffusion with new $\A{t}$ and $\Xtilde{t-1}$ in which the node $u$ is deleted.

\textbf{Discussion on other types of dynamic GNNs.}
There is another type of dynamic graphs, called traffic networks (e.g., roads), where the structure of a graph is static while node features vary over time.
For learning traffic networks, many models have been proposed~\cite{yu2018spatio,guo2019attention,wu2020connecting}, and they are used to do traffic forecasting.
However, those methods are our of the scope of the problem setting of \method because \method aims to augment a discrete-time dynamic graph where the structure of a graph changes over time, not remains static.
One promising future research direction is to devise an augmentation method for such traffic networks, which synthetically generates node features over time or forms temporal graph snapshots based on node feature similarities.
Another research direction is to expand our method to dynamic signed graphs~\cite{raghavendra2022signed} using signed graph diffusion~\cite{DBLP:conf/icdm/JungJSK16,jung2022signed}.

\section{Additional Experiments}
\label{sec:appendix:experiments}

\subsection{Effect of Hyperparameters}
We additionally conduct experiments to investigate the effect of hyperparameters of \method.
Figure~\ref{fig:effect_ab} demonstrates the effect of restart probability $\alpha$ and time travel probability $\beta$ where $0.1 \leq \alpha, \beta \leq 0.5$ in the BitcoinAlpha and the Brain datasets.
As shown in the figure, the effect of $\alpha$ and $\beta$ depends on datasets and models (similar in other datasets).
Figure~\ref{fig:effect_gamma:link} shows the results of link prediction in the RedditBody dataset according to values of $\gamma$.
Similar to the BitcoinAlpha and the WikiElec dataset, it is crucial to properly adjust $\gamma$ for the performance.

\begin{figure*}[t]
    \vspace{-7mm}
    \hspace{-2mm}
    \subfigure[GCN+\method on\\{\color{white}--}BitcoinAlpha (AUC)]{
        \label{fig:ab:alpha:gcn}
        \includegraphics[width=0.16\linewidth]{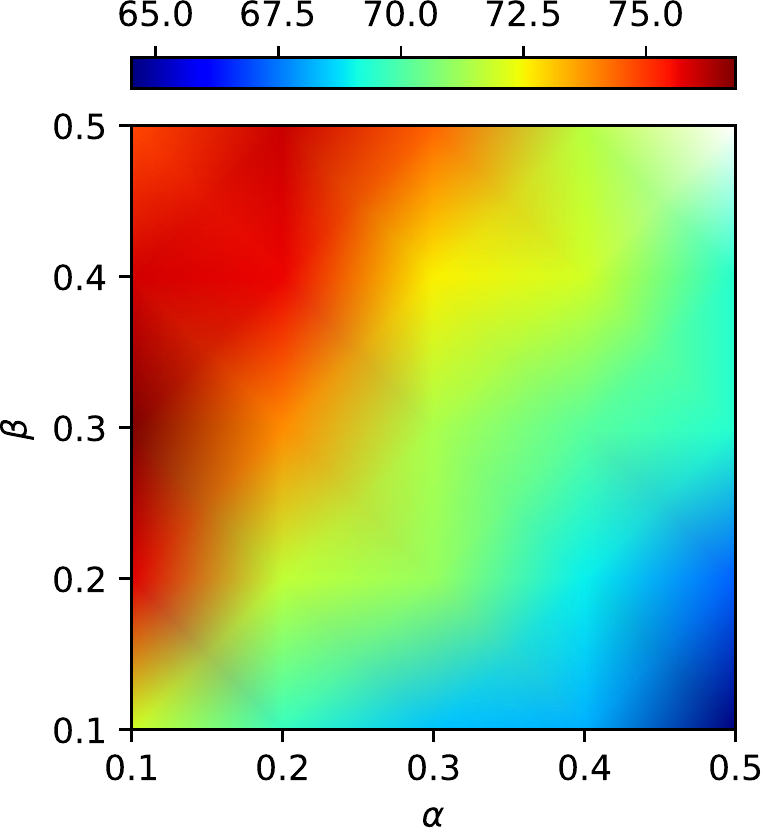}
    }
    \hspace{-2mm}
    \subfigure[GCRN+\method on\\{\color{white}--}BitcoinAlpha (AUC)]{
        \label{fig:ab:alpha:gcrn}
        \includegraphics[width=0.16\linewidth]{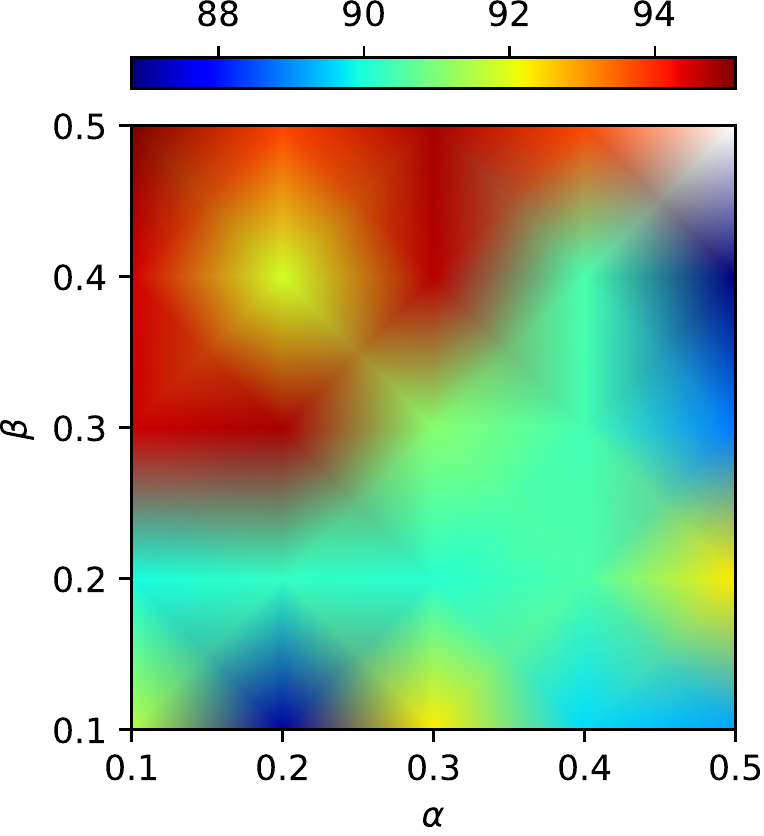}
    }
    \hspace{-2mm}
    \subfigure[EGCN+\method on\\{\color{white}--}BitcoinAlpha (AUC)]{
        \label{fig:ab:alpha:egcn}
        \includegraphics[width=0.16\linewidth]{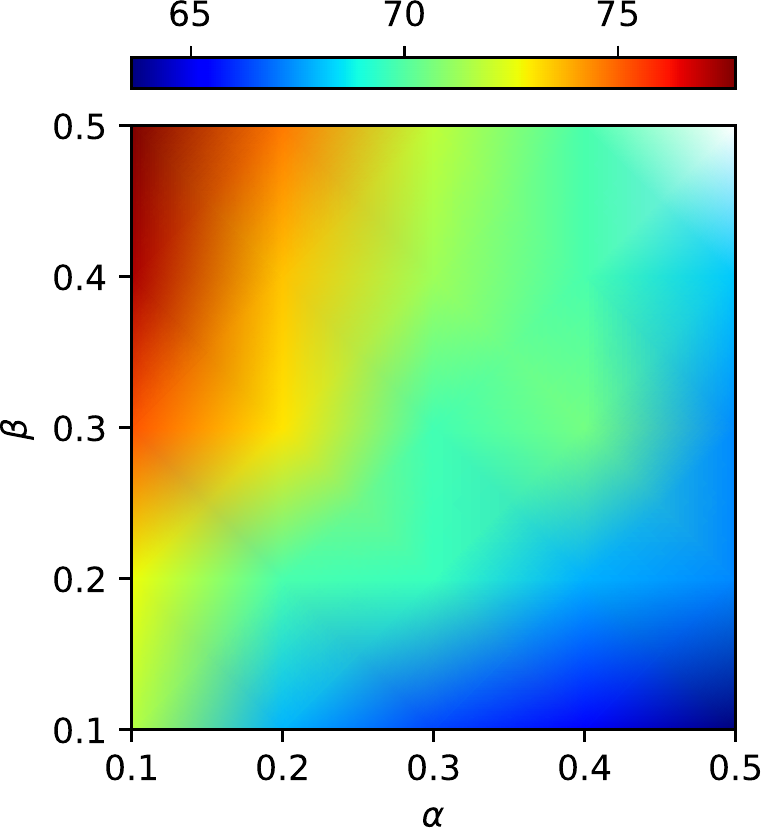}
    }
    \hspace{-2mm}
    \subfigure[GCN+\method on\\{\color{white}---}Brain (Macro F1)]{
        \label{fig:ab:brain:gcn}
        \includegraphics[width=0.16\linewidth]{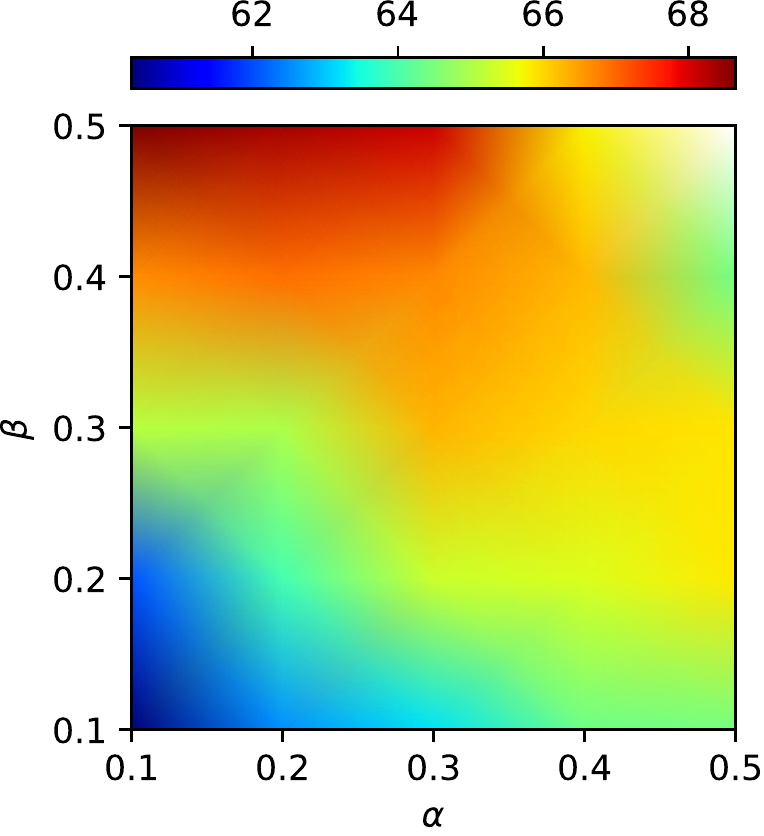}
    }
    \hspace{-2mm}
    \subfigure[GCRN+\method on\\{\color{white}---}Brain (Macro F1)]{
        \label{fig:ab:brain:gcrn}
        \includegraphics[width=0.16\linewidth]{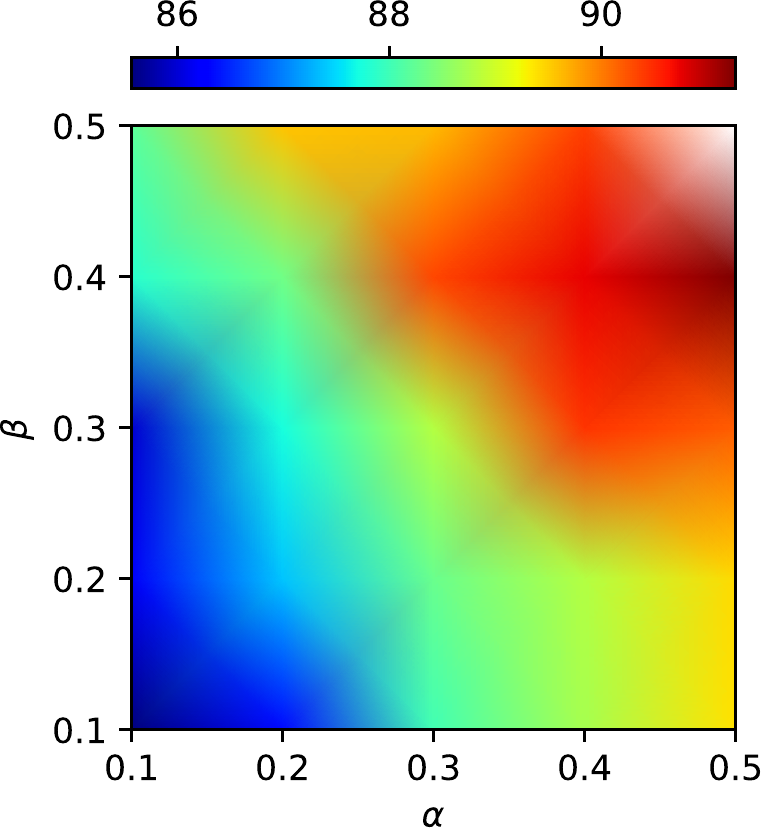}
    }
    \hspace{-2mm}
    \subfigure[EGCN+\method on\\{\color{white}---}Brain (Macro F1)]{
        \label{fig:ab:brain:egcn}
        \includegraphics[width=0.16\linewidth]{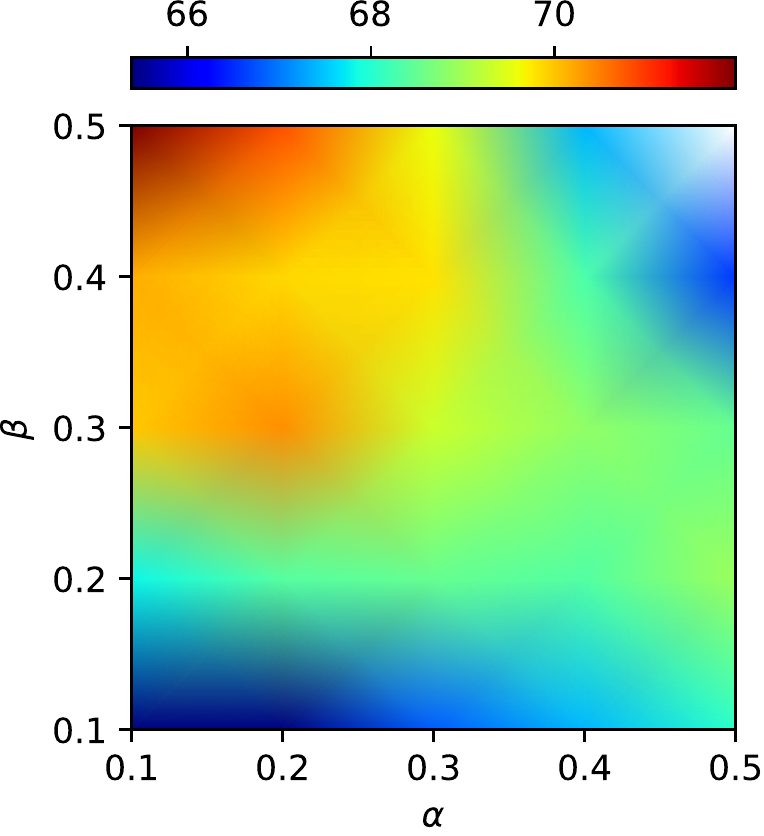}
    }
    \hspace{-2mm}
    \caption{
        \label{fig:effect_ab}
        Effect of restart probability $\alpha$ and time travel probability $\beta$ on the BitcoinAlpha and the Brain datasets.
    }
\end{figure*}
\begin{figure}[t]
    \vspace{-2mm}
    {\centering
        \includegraphics[width=1.0\linewidth]{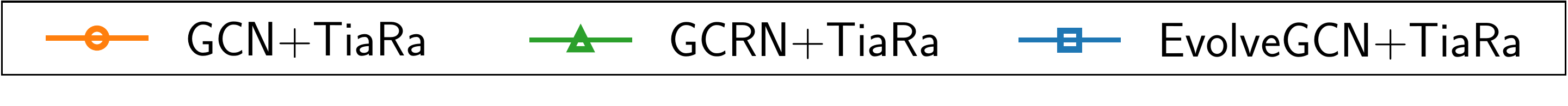} \\
        \hspace{18mm}
        \subfigure[RedditBody]{
            \label{fig:gamma:redditbody}
            \centering
            \includegraphics[width=0.45\linewidth]{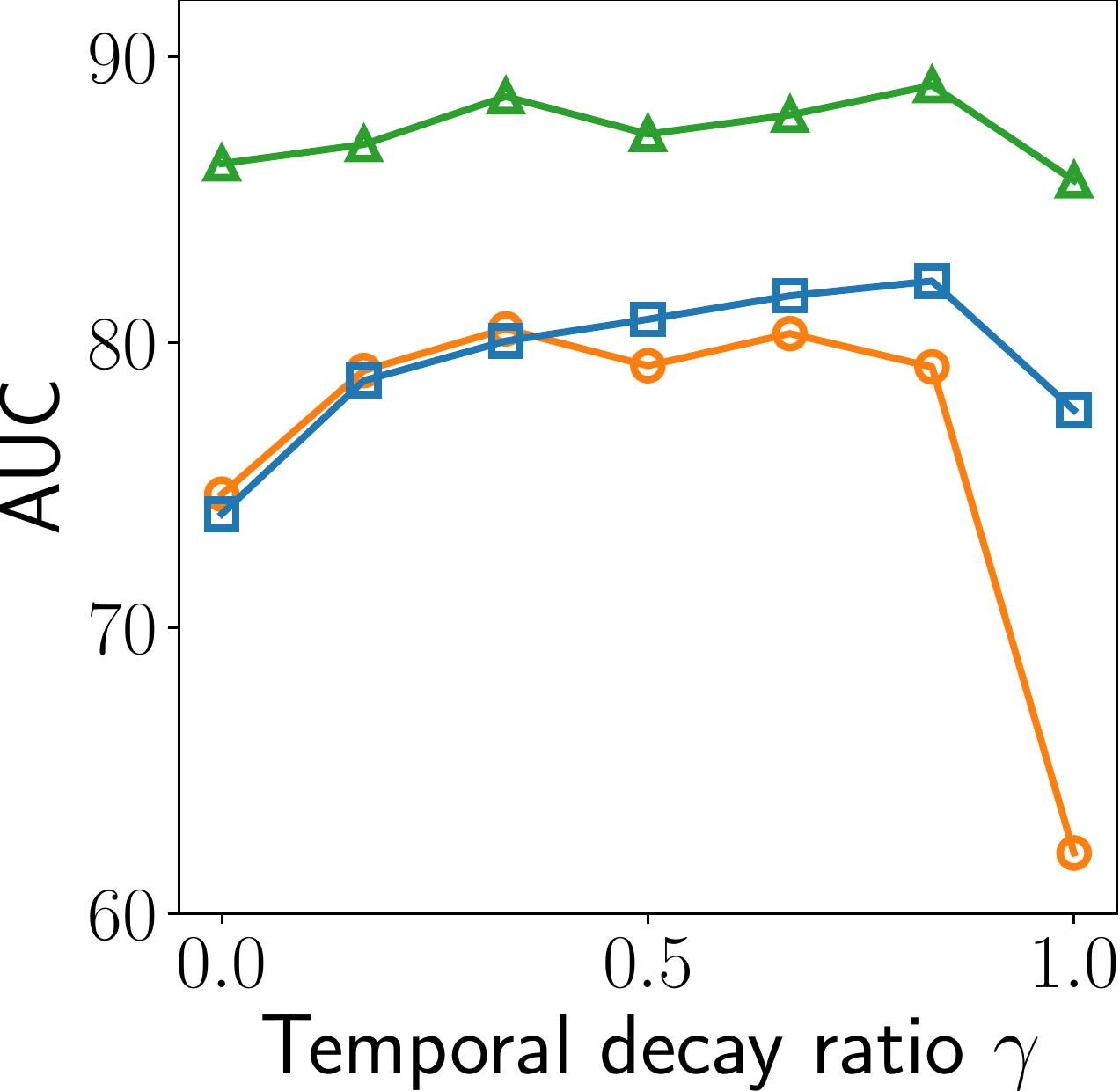}
        }
    }
    \caption{
        \label{fig:effect_gamma:link}
        Effect of the temporal decay ratio $\gamma$ w.r.t link prediction accuracy in the RedditBody dataset.
    }
    \vspace{-2mm}
\end{figure}

\section{Experimental Setting}
\label{sec:appendix:setting}
We describe detailed settings for our experiments including datasets, computing environment, hyperparemters, etc.

\subsection{Detailed Information of Datasets}
We use the following datasets for the temporal link prediction task:
\begin{itemize}
    \item {
        \textbf{BitcoinAlpha}\footnote{https://snap.stanford.edu/data/soc-sign-bitcoin-alpha.html}:
        It is a who-trusts-whom network of people who trade using Bitcoin on platforms called Bitcoin Alpha~\cite{kumar2018rev2}.
        We use only edges and their time stamps without rating in this work.
    }
    \item {
        \textbf{WikiElec}\footnote{https://snap.stanford.edu/data/wiki-Elec.html}:
        It is a voting network for Wikipedia elections~\cite{leskovec2010predicting}.
    }
    \item {
        \textbf{RedditBody}\footnote{https://snap.stanford.edu/data/soc-RedditHyperlinks.html}:
        This dataset is a hyperlink network representing connections between two posts, called sub-reddits, where a sub-reddit is a community on the Reddit~\cite{kumar2018community}.
    }
\end{itemize}

For the node classification task, we utilize the following datasets\footnote{\label{star_data}https://tinyurl.com/y67ywq6j}  used in~\cite{xu2019spatio}:
\begin{itemize}
    \item {
        \textbf{Brain}:
        This dataset is a network of brain issues where a node represents the tidy cube of brain tissue, and two nodes are connected if they show similar degree of activation.
        \citeauthor{xu2019spatio} applied PCA to the functional magnetic resonance imaging (fMRI) data to generate initial node features of $20$ dimensions.
    }
    \item {
        \textbf{DBLP-3} and \textbf{DBLP-5}:
        These datasets were extracted from the DBLP, and they are co-authorship networks where nodes represent authors.
        The authors in DBLP-3 and DBLP-5 are from three and five research areas, respectively.
        \citeauthor{xu2019spatio} used titles and abstracts of papers to produce node features by using word2vec where the features dimension is $100$.
    }
    \item {
        \textbf{Reddit}:
        This is a post network where a node is a post and two posted are connected if they share similar keywords.
        \citeauthor{xu2019spatio} applied word2vec to comments of a post to extract its node features of $20$ dimensions.
    }
\end{itemize}

\subsection{Computing Environment and Implementation}
We describe detailed specifications of computing environment and implementation used for experiments.
\begin{itemize}
    \item {CPU: Intel(R) Xeon(R) Silver 4215R CPU @ 3.20GHz}
    \item {GPU: NVIDIA GeForce RTX 3090 24GB}
    \item {CUDA: 11.3, Python:  3.10.4, PyTorch: 1.12.0}
\end{itemize}

\subsection{Baseline GNN Models}
\begin{itemize}
    \item {
        \textbf{GCN}~\cite{kipf2017semi}:
        This is a popular GNN model for a static graph, and based on graph convolutional operations.
        GCN first symmetrically normalizes an self-looped adjacency matrix $\mat{A}$ (i.e., self-loops are added) as $\Atilde{} = \mat{D}^{-\frac{1}{2}} \mat{A} \mat{D}^{-\frac{1}{2}}$ where $\mat{D}$ is a diagonal degree matrix of $\mat{A}$.
        After the normalization, it stacks a layer $\sigma(\Atilde{} \mat{X} \mat{W})$ given node feature matrix $\mat{X}$ and trainable parameter $\mat{W}$ where $\sigma(\cdot)$ is a non-linear activation fuction such as ReLU.
        The hyperparameters are hidden dimension, dropout ratio, and the number of layers.
    }
    \item {
        \textbf{GCRN}\footnote{\label{tgt} \url{https://github.com/benedekrozemberczki/pytorch_geometric_temporal}}~\cite{seo2018structured}:
        This combines GNN and RNN to learn dynamic graphs.
        In general, each operation in the RNN cell is based on matrix multiplications such as $\mat{W}_{x}\mat{X}_{t}$ and $\mat{W}_{h}\mat{H}_{t}$ where $\mat{W}_{*}$ is a parameter, $\mat{X}_{t}$ is an input feature, and $\mat{H}_{t}$ is a hidden state at time $t$.
        GCRN extends such matrix multiplications by using graph convolution as $\mat{W} *_{g} \mat{X}$ where $*_{g}$ is a graph convolution operation.
        Note that \citeauthor{seo2018structured} employed a GNN based on Chebyshev polynomial~\cite{cnn_graph} for $*_{g}$ while we use a GCN~\cite{kipf2017semi}, described in the above, because the GCN shows better performance, and it is more lightweight than the Chebyshev network.
        The hyperparameters are hidden dimension, the number of layers, dropout ratio, and the type of RNN (LSTM or GRU).
    }
    \item {
        \textbf{EvolveGCN}~\cite{ParejaDCMSKKSL20}: It uses an RNN to produce  parameters $\mat{W}_{t}$ at each time step $t$, which is used as GCN parameters while being treated as a hidden state of the RNN.
        The authors interpret the process as the evolution of GCN parameters over time.
        The hyperparameters are hidden dimension, the number of layers, dropout ratio, and the type of RNN (LSTM or GRU).
    }
\end{itemize}

\subsection{Baseline Augmentation Methods}
\begin{itemize}
    \item {%
        \textbf{DropEdge}\footnote{\url{https://github.com/dmlc/dgl}}~\cite{Rong2020DropEdge}: It is a drop-based method randomly removing edges at each epoch.
        The hyperparameter is a dropedge ratio b.t.w. $0$ and $1$.
    }
    \item {
        \textbf{GDC}~\cite{klicpera2019diffusion}: It is a diffusion-based method where we use PPR for GDC as our approach is based on random walks.
        The hyperparameter is $\alpha$, a restart probability.
    }
    \item {
        \textbf{Merge}: It is a simple baseline merging adjacency matrices from time $1$ to $t$ when training a model at time $t$. The overlapped entries are processed as $1$.
        This method has no hyperparameters.
    }
\end{itemize}

\subsection{Implementation Details}

\textbf{Details on the temporal link prediction task.}
The datasets for this task do not contain initial node features; thus, we first preprocess random features $\F{t}$ of $32$ dimensions for each time $t$, which are sampled from a {standard normal distribution}, and use them for all models as initial features.
One may use advanced methods~\cite{DBLP:conf/kdd/YooJJK22} for estimating node features when they are partially unavailable in a graph.
For each dataset, we aggregate multiple edges within a specific time range, called \texttt{\small time\_aggregation}, to represent the data as a discrete-time dynamic graph (i.e., a temporal sequence of graph snapshots), and use {$1,200,000$} (seconds) for \texttt{\small  time\_aggregation}.
We also make each graph snapshot undirected since most GNNs require undirected graphs as input although this process is not mandatory for  \method.

To evaluate this task, we consider existing edges as positive samples, and extract negative edge samples as much as positive ones for each graph snapshot.
More specifically, let $d_{t,u}$ be the out-degree of node $u$ in a graph snapshot $\G{t}$.
Then, for each node $u$, we randomly sample $d_{t,u}$ nodes, $V_{t,u}=\{v_1, \cdots, v_{d_{t,u}}\}$, from other unconnected nodes except its neighbors, and $(u, v_i)$ s.t. $v_i\in V_{t,u}$ is considered as a negative sample.

As a decoder of the task, we use a simple MLP (i.e., \texttt{FC}-\texttt{ReLU}-\texttt{FC}) to predict whether a given pair of nodes forms an edge or not at a specific time step.
Given $(u, v)$ and time $t+1$, we concatenate two node embeddings of $u$ and $v$, which were produced at time $t$, as the input of the decoder.
The decoder outputs scores (or logits) for positive and negative classes, and forwards the scores to the cross entropy loss function.
The total loss is the mean of the losses of all samples across all time steps within a training (or validation/test) set.
The input dimension of the decoder is $64$, and the output dimension is $2$.
We set the hidden dimension of the decoder to $32$.
We use the same decoder for all models.

\def\arraystretch{1.2}
\setlength{\tabcolsep}{7.9pt}
\begin{table*}[!bp]
    \caption{Searched hyperparameters of GCN in the temporal link prediction task}
    \label{tab:hyperparam:gcn:link}
    \small
    \begin{tabular}{@{}ccccccccccc@{}}
        \hline
        \toprule
        \textbf{Augmentation} & \textbf{Dataset} & \textbf{\# of layers} & \textbf{\begin{tabular}[c]{@{}c@{}}learing \\ rate\end{tabular}} & \textbf{\begin{tabular}[c]{@{}c@{}}RNN \\ type\end{tabular}} & \textbf{\begin{tabular}[c]{@{}c@{}}dropout \\ ratio\end{tabular}} & \textbf{\begin{tabular}[c]{@{}c@{}}dropedge \\ ratio\end{tabular}} & \textbf{$\alpha$} & \textbf{$\beta$} & \textbf{$\epsilon$} & \textbf{\begin{tabular}[c]{@{}c@{}}symmetric \\ trick\end{tabular}} \\
        \midrule
        \multirow{3}{*}{\textsc{None}}
        & BitcoinAlpha & 3 & 0.05 & \multirow{3}{*}{-} & 0 & \multirow{3}{*}{-} & \multirow{3}{*}{-} & \multirow{3}{*}{-} & \multirow{3}{*}{-} & \multirow{3}{*}{-} \\
        & WikiElec & 3 & 0.02 & & 0 & & & & & \\
        & RedditBody & 3 & 0.02 & & 0 & & & & & \\
        \midrule
        \multirow{3}{*}{\textsc{DropEdge}} & BitcoinAlpha & 3 & 0.05 & \multirow{3}{*}{-} & 0 & 0.6 & \multirow{3}{*}{-} & \multirow{3}{*}{-} & \multirow{3}{*}{-} & \multirow{3}{*}{-} \\
        & WikiElec & 3 & 0.02 & & 0 & 0.9 & & & & \\
        & RedditBody & 3 & 0.02 & & 0 & 0.9 & & & & \\
        \midrule
        \multirow{3}{*}{\textsc{GDC}} & BitcoinAlpha & 3 & 0.05 & \multirow{3}{*}{-} & 0 & \multirow{3}{*}{-} & 0.4 & \multirow{3}{*}{-} & \multirow{3}{*}{0.001} & \multirow{3}{*}{Off} \\
        & WikiElec & 3 & 0.02 & & 0.5 & & 0.2 & & & \\
        & RedditBody & 3 & 0.02 & & 0.5 & & 0.2 & & & \\
        \midrule
        \multirow{3}{*}{\textsc{Merge}} & BitcoinAlpha & 3 & 0.05 & \multirow{3}{*}{-} & 0 & \multirow{3}{*}{-} & \multirow{3}{*}{-} & \multirow{3}{*}{-} & \multirow{3}{*}{-} & \multirow{3}{*}{-} \\
        & WikiElec & 3 & 0.05 & & 0 & & & & & \\
        & RedditBody & 2 & 0.02 & & 0 & & & & & \\
        \midrule
        \multirow{3}{*}{\textsc{Tiara}} & BitcoinAlpha & 3 & 0.05 & \multirow{3}{*}{-} & 0 & \multirow{3}{*}{-} & 0.05 & 0.2 & \multirow{3}{*}{0.001} & \multirow{3}{*}{Off} \\
        & WikiElec & 3 & 0.02 & & 0 & & 0.1 & 0.2 & & \\
        & RedditBody & 3 & 0.02 & & 0.5 & & 0.3 & 0.2 & & \\ 
        \bottomrule
        \hline
    \end{tabular}
\end{table*}

\textbf{Details on the node classification task.}
The original dynamic graphs have features, and they are represented as a sequence of graph snapshots.
As in \cite{xu2019spatio}, we randomly split nodes for a train/validation/test set with the 7:1:2 ratio.
In this task, as a decoder, we also use a simple MLP having the same architecture as the aforementioned decoder except that this task's decoder receives the embedding of each node on the last time step (note that labels of this task do not have temporal information).
The decoder outputs $L$ scores where $L$ is the number of node labels (or classes).
We also utilize the cross entropy loss function in this task.
The input dimension of the decoder is the feature dimension $d$ of $\F{t}$, and the output dimension is $L$.
We set the hidden dimension of the decoder to $32$.
We exploit the same decoder for all models.

\textbf{Symmetric trick.}
As suggested in~\cite{klicpera2019diffusion}, we also found that making $\Xtilde{t}$ symmetric can improve predictive performance in several settings, and call this a symmetric trick.
Since $\Xtilde{t}$ is not symmetric, we first make it symmetric as $\Xhat{t}\!=\!(\Xtilde{t}\!+\!\Xtilde{t}^{\top})/2$.
We then drop the weights of entries of $\Xhat{t}$ (i.e., replace non-zero elements to $1$), and the resulting matrix is denoted by $\Ahat{t}$.
Note $\Ahat{t}$ is not normalized.
Thus, we perform symmetric normalization on $\Ahat{t}$, i.e., $\Atilde{t} = \Dhat{t}^{-\frac{1}{2}} \Ahat{t} \Dhat{t}^{-\frac{1}{2}}$ where $\Dhat{t}$ is a diagonal degree matrix of $\Ahat{t}$.
After that, $\Atilde{t}$ is fed into a model as an augmented adjacency matrix.
We use this technique for graph diffusion-based methods such as \textsc{GDC} and \method.

\subsection{Hyperparameters}
We report the hyperparameter search bounds and the used configurations.
We summarize the information of hyperparameters that we tuned for experiments in Table~\ref{tab:hyperparam}.
\textsc{Common} indicates hyperparameters that commonly appear in GNN models or the optimizer.
We fix weight decay to $10^{-4}$, $K$ of $\method$ to $100$, and learning rate decay to $0.999$ for each epoch.
We fix the embedding dimension of GNN models to $d$ which is the feature dimension of $\F{t}$.
The search bound of each hyperparameters is described in Table~\ref{tab:search_bound}.
We summarize the final hyperparemters obtained by the search in Tables~\ref{tab:hyperparam:gcn:link}$\sim$\ref{tab:hyperparam:egcn:node}.

\def\arraystretch{1.2}
\setlength{\tabcolsep}{4pt}
\begin{table}[t!]
    \caption{Summary of hyperparameters}
    \label{tab:hyperparam}
    \small
    \begin{tabular}{@{}cll@{}}
        \toprule
        \textbf{Method} & \textbf{Hyperparameters} \\
        \midrule
        \textsc{Common} & learning rate, \# of layers, RNN type, dropout ratio \\
        \textsc{DropEdge} & dropedge ratio \\
        \textsc{GDC} & restart prob. $\alpha$, filtering threshold $\epsilon$, symmetric trick \\
        \method & restart prob. $\alpha$, time travel prob. $\beta$, \\
        & filtering threshold $\epsilon$, symmetric trick \\
        \bottomrule
    \end{tabular}
\end{table}

\def\arraystretch{1.2}
\setlength{\tabcolsep}{1.5pt}
\begin{table}[t!]
    \vspace{3mm}
    \caption{Search bounds of hyperparameters}
    \label{tab:search_bound}
    \small
    \begin{tabular}{@{}cll@{}}
        \toprule
        \textbf{Hyperparameters} & \textbf{Search bounds} \\
        \midrule
        learning rate & $\{0.01, 0.02, 0.05\}$ \\
        number of layers & $\lbrace 2, 3 \rbrace$ \\
        RNN type & $\lbrace \text{LSTM}, \text{GRU} \rbrace$ \\
        dropout ratio & $[0, 0.5]$ by $0.05$ \\
        dropedge ratio & $[0.1, 0.9]$ by $0.1$ \\
        $\alpha$ of GDC & {$\{0.01, 0.02, 0.05, 0.1, 0.15, 0.2, 0.3, 0.4, 0.5\}$} \\
        $\alpha$, $\beta$ of \method & $\lbrace (\alpha, \beta)\rbrace$ where $\alpha \in (0,1)$, $\beta \in (0,1),$ \\
        & and $\alpha + \beta \in (0,1)$ \\
        $\epsilon$ & $\{0.0001, 0.001, 0.01\}$ \\
        \bottomrule
    \end{tabular}
\end{table}

\def\arraystretch{1.2}
\setlength{\tabcolsep}{7.9pt}
\begin{table*}[t]
    \caption{Searched hyperparameters of GCN in the node classification task}
    \label{tab:hyperparam:gcn:node}
    \small
    \begin{tabular}{@{}ccccccccccc@{}}
        \hline
        \toprule
        \textbf{Augmentation} & \textbf{Dataset} & \textbf{\# of layers} & \textbf{\begin{tabular}[c]{@{}c@{}}learing \\ rate\end{tabular}} & \textbf{\begin{tabular}[c]{@{}c@{}}RNN \\ type\end{tabular}} & \textbf{\begin{tabular}[c]{@{}c@{}}dropout \\ ratio\end{tabular}} & \textbf{\begin{tabular}[c]{@{}c@{}}dropedge \\ ratio\end{tabular}} & \textbf{$\alpha$} & \textbf{$\beta$} & \textbf{$\epsilon$} & \textbf{\begin{tabular}[c]{@{}c@{}}symmetric \\ trick\end{tabular}} \\
        \midrule
        \multirow{4}{*}{\textsc{None}} & Brain & 2 & 0.01 & \multirow{4}{*}{-} & 0 & \multirow{4}{*}{-} & \multirow{4}{*}{-} & \multirow{4}{*}{-} & \multirow{4}{*}{-} & \multirow{4}{*}{-} \\
        & Reddit & 3 & 0.05 & & 0 & & & & & \\
        & DBLP3 & 2 & 0.05 & & 0 & & & & & \\
        & DBLP5 & 2 & 0.01 & & 0 & & & & & \\
        \midrule
        \multirow{4}{*}{\textsc{DropEdge}} & Brain & 2 & 0.01 & \multirow{4}{*}{-} & 0.5 & 0.1 & \multirow{4}{*}{-} & \multirow{4}{*}{-} & \multirow{4}{*}{-} & \multirow{4}{*}{-} \\
        & Reddit & 3 & 0.05 & & 0.5 & 0.8 & & & & \\
        & DBLP3 & 2 & 0.05 & & 0 & 0.3 & & & & \\
        & DBLP5 & 2 & 0.05 & & 0.5 & 0.6 & & & & \\
        \midrule
        \multirow{4}{*}{\textsc{GDC}} & Brain & 2 & 0.01 & \multirow{4}{*}{-} & 0 & \multirow{4}{*}{-} & 0.5 & \multirow{4}{*}{-} & \multirow{4}{*}{0.001} & \multirow{4}{*}{Off} \\
        & Reddit & 3 & 0.05 & & 0 & & 0.3 & & & \\
        & DBLP3 & 2 & 0.05 & & 0 & & 0.3 & & & \\
        & DBLP5 & 2 & 0.05 & & 0 & & 0.05 & & & \\
        \midrule
        \multirow{4}{*}{\textsc{Merge}} & Brain & 2 & 0.02 & \multirow{4}{*}{-} & 0 & \multirow{4}{*}{-} & \multirow{4}{*}{-} & \multirow{4}{*}{-} & \multirow{4}{*}{-} & \multirow{4}{*}{-} \\
        & Reddit & 3 & 0.02 & & 0 & & & & & \\
        & DBLP3 & 2 & 0.01 & & 0 & & & & & \\
        & DBLP5 & 2 & 0.05 & & 0 & & & & & \\
        \midrule
        \multirow{4}{*}{\method} & Brain & 2 & 0.01 & \multirow{4}{*}{-} & 0 & \multirow{4}{*}{-} & 0.1 & 0.8 & \multirow{4}{*}{0.001} & Off \\
        & Reddit & 3 & 0.05 & & 0 & & 0.2 & 0.4 & & Off \\
        & DBLP3 & 2 & 0.05 & & 0.5 & & 0.1 & 0.4 & & On \\
        & DBLP5 & 2 & 0.01 & & 0.5 & & 0.001 & 0.009 & & On \\ 
        \bottomrule
        \hline
    \end{tabular}
\end{table*}

\def\arraystretch{1.2}
\setlength{\tabcolsep}{6.6pt}
\begin{table*}[t]
    \caption{Searched hyperparameters of GCRN in the temporal link prediction task}
    \label{tab:hyperparam:gcrn:link}
    \small
    \begin{tabular}{@{}ccccccccccc@{}}
        \hline
        \toprule
        \textbf{Augmentation} & \textbf{Dataset} & \textbf{\# of layers} & \textbf{\begin{tabular}[c]{@{}c@{}}learing \\ rate\end{tabular}} & \textbf{\begin{tabular}[c]{@{}c@{}}RNN \\ type\end{tabular}} & \textbf{\begin{tabular}[c]{@{}c@{}}dropout \\ ratio\end{tabular}} & \textbf{\begin{tabular}[c]{@{}c@{}}dropedge \\ ratio\end{tabular}} & \textbf{$\alpha$} & \textbf{$\beta$} & \textbf{$\epsilon$} & \textbf{\begin{tabular}[c]{@{}c@{}}symmetric \\ trick\end{tabular}} \\
        \midrule
        \multirow{3}{*}{\textsc{None}} & BitcoinAlpha & 3 & 0.05 & LSTM & 0 & \multirow{3}{*}{-} & \multirow{3}{*}{-} & \multirow{3}{*}{-} & \multirow{3}{*}{-} & \multirow{3}{*}{-} \\
        & WikiElec & 2 & 0.05 & LSTM & 0 & & & & & \\
        & RedditBody & 3 & 0.02 & LSTM & 0 & & & & & \\
        \midrule
        \multirow{3}{*}{\textsc{DropEdge}} & BitcoinAlpha & 3 & 0.05 & LSTM & 0.5 & 0.5 & \multirow{3}{*}{-} & \multirow{3}{*}{-} & \multirow{3}{*}{-} & \multirow{3}{*}{-} \\
        & WikiElec & 2 & 0.05 & LSTM & 0.5 & 0.8 & & & & \\
        & RedditBody & 3 & 0.02 & LSTM & 0 & 0.7 & & & & \\
        \midrule
        \multirow{3}{*}{\textsc{GDC}} & BitcoinAlpha & 3 & 0.05 & LSTM & 0.5 & \multirow{3}{*}{-} & 0.2 & \multirow{3}{*}{-} & \multirow{3}{*}{0.001} & \multirow{3}{*}{On} \\
        & WikiElec & 2 & 0.05 & LSTM & 0.5 & & 0.1 & & & \\
        & RedditBody & 3 & 0.02 & LSTM & 0.5 & & 0.3 & & & \\
        \midrule
        \multirow{3}{*}{\textsc{Merge}} & BitcoinAlpha & 3 & 0.05 & LSTM & 0.5 & \multirow{3}{*}{-} & \multirow{3}{*}{-} & \multirow{3}{*}{-} & \multirow{3}{*}{-} & \multirow{3}{*}{-} \\
        & WikiElec & 3 & 0.05 & LSTM & 0.5 & & & & & \\
        & RedditBody & 3 & 0.01 & LSTM & 0 & & & & & \\
        \midrule
        \multirow{3}{*}{\method} & BitcoinAlpha & 3 & 0.05 & LSTM & 0.5 & \multirow{3}{*}{-} & 0.1 & 0.3 & \multirow{3}{*}{0.001} & \multirow{3}{*}{On} \\
        & WikiElec & 2 & 0.05 & LSTM & 0.5 & & 0.1 & 0.3 & & \\
        & RedditBody & 3 & 0.02 & LSTM & 0.1 & & 0.1275 & 0.7225 & & \\ 
        \bottomrule
        \hline
    \end{tabular}
\end{table*}

\def\arraystretch{1.2}
\setlength{\tabcolsep}{8.4pt}
\begin{table*}[t]
    \caption{Searched hyperparameters of GCRN in the node classification task}
    \label{tab:hyperparam:gcrn:node}
    \small
    \begin{tabular}{@{}ccccccccccc@{}}
        \hline
        \toprule
        \textbf{Augmentation} & \textbf{Dataset} & \textbf{\# of layers} & \textbf{\begin{tabular}[c]{@{}c@{}}learing \\ rate\end{tabular}} & \textbf{\begin{tabular}[c]{@{}c@{}}RNN \\ type\end{tabular}} & \textbf{\begin{tabular}[c]{@{}c@{}}dropout \\ ratio\end{tabular}} & \textbf{\begin{tabular}[c]{@{}c@{}}dropedge \\ ratio\end{tabular}} & \textbf{$\alpha$} & \textbf{$\beta$} & \textbf{$\epsilon$} & \textbf{\begin{tabular}[c]{@{}c@{}}symmetric \\ trick\end{tabular}} \\
        \midrule
        \multirow{4}{*}{\textsc{None}} & Brain & 2 & 0.02 & GRU & 0 & \multirow{4}{*}{-} & \multirow{4}{*}{-} & \multirow{4}{*}{-} & \multirow{4}{*}{-} & \multirow{4}{*}{-} \\
        & Reddit & 2 & 0.01 & LSTM & 0 & & & & & \\
        & DBLP3 & 3 & 0.02 & GRU & 0 & & & & & \\
        & DBLP5 & 2 & 0.01 & GRU & 0 & & & & & \\
        \midrule
        \multirow{4}{*}{\textsc{DropEdge}} & Brain & 2 & 0.02 & GRU & 0 & 0.6 & \multirow{4}{*}{-} & \multirow{4}{*}{-} & \multirow{4}{*}{-} & \multirow{4}{*}{-} \\
        & Reddit & 2 & 0.05 & LSTM & 0.5 & 0.1 & & & & \\
        & DBLP3 & 2 & 0.02 & LSTM & 0.5 & 0.8 & & & & \\
        & DBLP5 & 2 & 0.01 & GRU & 0.5 & 0.8 & & & & \\
        \midrule
        \multirow{4}{*}{\textsc{GDC}} & Brain & 2 & 0.01 & LSTM & 0 & \multirow{4}{*}{-} & 0.4 & \multirow{4}{*}{-} & \multirow{4}{*}{0.001} & \multirow{4}{*}{Off} \\
        & Reddit & 2 & 0.02 & LSTM & 0 & & 0.1 & & & \\
        & DBLP3 & 2 & 0.02 & LSTM & 0 & & 0.5 & & & \\
        & DBLP5 & 2 & 0.02 & LSTM & 0 & & 0.3 & & & \\
        \midrule
        \multirow{4}{*}{\textsc{Merge}} & Brain & 2 & 0.02 & GRU & 0 & \multirow{4}{*}{-} & \multirow{4}{*}{-} & \multirow{4}{*}{-} & \multirow{4}{*}{-} & \multirow{4}{*}{-} \\
        & Reddit & 2 & 0.01 & LSTM & 0 & & & & & \\
        & DBLP3 & 2 & 0.02 & LSTM & 0 & & & & & \\
        & DBLP5 & 2 & 0.01 & GRU & 0 & & & & & \\
        \midrule
        \multirow{4}{*}{\method} & Brain & 2 & 0.01 & LSTM & 0 & \multirow{4}{*}{-} & 0.5 & 0.4 & \multirow{4}{*}{0.001} & \multirow{4}{*}{Off} \\
        & Reddit & 2 & 0.02 & LSTM & 0 & & 0.4 & 0.5 & & \\
        & DBLP3 & 2 & 0.02 & LSTM & 0 & & 0.4 & 0.5 & & \\
        & DBLP5 & 2 & 0.02 & LSTM & 0.5 & & 0.1 & 0.8 & & \\ 
        \bottomrule
        \hline
    \end{tabular}
\end{table*}

\def\arraystretch{1.2}
\setlength{\tabcolsep}{8pt}
\begin{table*}[t]
    \caption{Searched hyperparameters of EvolveGCN (EGCN) in the temporal link prediction task}
    \label{tab:hyperparam:egcn:link}
    \small
    \begin{tabular}{@{}ccccccccccc@{}}
        \hline
        \toprule
        \textbf{Augmentation} & \textbf{Dataset} & \textbf{\# of layers} & \textbf{\begin{tabular}[c]{@{}c@{}}learing \\ rate\end{tabular}} & \textbf{\begin{tabular}[c]{@{}c@{}}RNN \\ type\end{tabular}} & \textbf{\begin{tabular}[c]{@{}c@{}}dropout \\ ratio\end{tabular}} & \textbf{\begin{tabular}[c]{@{}c@{}}dropedge \\ ratio\end{tabular}} & \textbf{$\alpha$} & \textbf{$\beta$} & \textbf{$\epsilon$} & \textbf{\begin{tabular}[c]{@{}c@{}}symmetric \\ trick\end{tabular}} \\ 
        \midrule
        \multirow{3}{*}{\textsc{None}} & BitcoinAlpha & 3 & 0.01 & LSTM & 0 & \multirow{3}{*}{-} & \multirow{3}{*}{-} & \multirow{3}{*}{-} & \multirow{3}{*}{-} & \multirow{3}{*}{-} \\
        & WikiElec & 2 & 0.05 & LSTM & 0 & & & & & \\
        & RedditBody & 3 & 0.02 & LSTM & 0 & & & & & \\
        \midrule
        \multirow{3}{*}{\textsc{DropEdge}} & BitcoinAlpha & 3 & 0.01 & LSTM & 0 & 0.8 & \multirow{3}{*}{-} & \multirow{3}{*}{-} & \multirow{3}{*}{-} & \multirow{3}{*}{-} \\
        & WikiElec & 2 & 0.05 & LSTM & 0 & 0.1 & & & & \\
        & RedditBody & 3 & 0.02 & LSTM & 0 & 0.5 & & & & \\
        \midrule
        \multirow{3}{*}{\textsc{GDC}} & BitcoinAlpha & 3 & 0.01 & LSTM & 0.5 & \multirow{3}{*}{-} & 0.2 & \multirow{3}{*}{-} & \multirow{3}{*}{0.001} & \multirow{3}{*}{On} \\
        & WikiElec & 2 & 0.05 & LSTM & 0 & & 0.1 & & & \\
        & RedditBody & 3 & 0.02 & LSTM & 0.5 & & 0.2 & & & \\
        \midrule
        \multirow{3}{*}{\textsc{Merge}} & BitcoinAlpha & 3 & 0.05 & LSTM & 0.5 & \multirow{3}{*}{-} & \multirow{3}{*}{-} & \multirow{3}{*}{-} & \multirow{3}{*}{-} & \multirow{3}{*}{-} \\
        & WikiElec & 3 & 0.05 & GRU & 0 & & & & & \\
        & RedditBody & 3 & 0.05 & LSTM & 0 & & & & & \\
        \midrule
        \multirow{3}{*}{\method} & BitcoinAlpha & 3 & 0.01 & LSTM & 0.5 & \multirow{3}{*}{-} & 0.1 & 0.4 & \multirow{3}{*}{0.001} & \multirow{3}{*}{On} \\
        & WikiElec & 2 & 0.05 & LSTM & 0.5 & & 0.3 & 0.4 & & \\
        & RedditBody & 3 & 0.02 & LSTM & 0.5 & & 0.1 & 0.2 & & \\ 
        \bottomrule
    \end{tabular}
\end{table*}

\def\arraystretch{1.2}
\setlength{\tabcolsep}{7.4pt}
\begin{table*}[!t]
    \caption{Searched hyperparameters of EvolveGCN (EGCN) in the node classification task}
    \label{tab:hyperparam:egcn:node}
    \small
    \begin{tabular}{@{}ccccccccccc@{}}
        \hline
        \toprule
        \textbf{Augmentation} & \textbf{Dataset} & \textbf{\# of layers} & \textbf{\begin{tabular}[c]{@{}c@{}}learing \\ rate\end{tabular}} & \textbf{\begin{tabular}[c]{@{}c@{}}RNN \\ type\end{tabular}} & \textbf{\begin{tabular}[c]{@{}c@{}}dropout \\ ratio\end{tabular}} & \textbf{\begin{tabular}[c]{@{}c@{}}dropedge \\ ratio\end{tabular}} & \textbf{$\alpha$} & \textbf{$\beta$} & \textbf{$\epsilon$} & \textbf{\begin{tabular}[c]{@{}c@{}}symmetric \\ trick\end{tabular}} \\
        \midrule
        \multirow{4}{*}{\textsc{None}} & Brain & 2 & 0.01 & LSTM & 0 & \multirow{4}{*}{-} & \multirow{4}{*}{-} & \multirow{4}{*}{-} & \multirow{4}{*}{-} & \multirow{4}{*}{-} \\
        & Reddit & 2 & 0.02 & LSTM & 0 & & & & & \\
        & DBLP3 & 2 & 0.02 & LSTM & 0 & & & & & \\
        & DBLP5 & 2 & 0.02 & LSTM & 0 & & & & & \\
        \midrule
        \multirow{4}{*}{\textsc{DropEdge}} & Brain & 2 & 0.01 & LSTM & 0 & 0.6 & \multirow{4}{*}{-} & \multirow{4}{*}{-} & \multirow{4}{*}{-} & \multirow{4}{*}{-} \\
        & Reddit & 2 & 0.02 & LSTM & 0 & 0.5 & & & & \\
        & DBLP3 & 2 & 0.02 & LSTM & 0.5 & 0.2 & & & & \\
        & DBLP5 & 2 & 0.02 & LSTM & 0.5 & 0.1 & & & & \\
        \midrule
        \multirow{4}{*}{\textsc{GDC}} & Brain & 2 & 0.01 & LSTM & 0 & \multirow{4}{*}{-} & 0.5 & \multirow{4}{*}{-} & \multirow{4}{*}{0.001} & \multirow{4}{*}{Off} \\
        & Reddit & 2 & 0.02 & LSTM & 0 & & 0.1 & & & \\
        & DBLP3 & 2 & 0.02 & LSTM & 0 & & 0.01 & & & \\
        & DBLP5 & 2 & 0.02 & LSTM & 0 & & 0.5 & & & \\
        \midrule
        \multirow{4}{*}{\textsc{Merge}} & Brain & 2 & 0.05 & LSTM & 0 & \multirow{4}{*}{-} & \multirow{4}{*}{-} & \multirow{4}{*}{-} & \multirow{4}{*}{-} & \multirow{4}{*}{-} \\
        & Reddit & 2 & 0.01 & LSTM & 0 & & & & & \\
        & DBLP3 & 2 & 0.01 & LSTM & 0 & & & & & \\
        & DBLP5 & 2 & 0.02 & GRU & 0 & & & & & \\
        \midrule
        \multirow{4}{*}{\method} & Brain & 2 & 0.01 & LSTM & 0 & \multirow{4}{*}{-} & 0.1 & 0.5 & \multirow{4}{*}{0.001} & Off \\
        & Reddit & 2 & 0.02 & LSTM & 0.05 & & 0.002 & 0.008 & & Off \\
        & DBLP3 & 2 & 0.02 & LSTM & 0.5 & & 0.1 & 0.5 & & On \\
        & DBLP5 & 2 & 0.02 & LSTM & 0.7 & & 0.1 & 0.8 & & Off \\
        \bottomrule
        \hline
    \end{tabular}
\end{table*}

\end{document}